%% file: paper.tex
\RequirePackage{fix-cm}
\documentclass[10pt,a4paper,logo]{paper}

\usepackage[all]{hypcap}
\usepackage[authoryear,round]{natbib}
\usepackage{relsize}
\usepackage{mathtools}
\usepackage{tikz}
\usepackage{centernot}
\usepackage{algorithm}
\usepackage{algpseudocode}
\usepackage{etoolbox}
\AtBeginEnvironment{algorithmic}{\small}
\usepackage{multirow}
\usepackage{threeparttable}
\usepackage{array,makecell}
\usepackage{placeins}
\usepackage{float}
\usepackage{wrapfig}
\usepackage{adjustbox}
\usepackage{anyfontsize}
\usepackage[capitalize,noabbrev]{cleveref}

\definecolor{title_blue}{HTML}{204899}
\definecolor{cite_blue}{HTML}{044dc1}
\definecolor{cite_purple}{HTML}{7406a7}
\hypersetup{
  colorlinks=true,
  citecolor=cite_blue,
  linkcolor=cite_purple,
  urlcolor=cite_purple,
  pdftitle={Diffusion Policy Improvement with Proposal-Conditioned Refinement Flows},
  pdfauthor={Junhyun Ha, Juho Lee, Byoungwoo Park}
}

\input{macro}

\makeatletter
\def\theHALG@line{\thealgorithm.\arabic{ALG@line}}
\patchcmd{\maketitle}
  {\@author \\ \footnotesize{\textbf{*}}Equal Contribution, \textsuperscript{$\dagger$}Corresponding Authors\par}
  {\@author\par}
  {}
  {\PackageError{preflow}{Could not remove template author legend}{Check paper.cls.}}
\makeatother
\renewcommand{\HorRule}{%
  \color{black}\makebox[\linewidth][l]{\rule{\textwidth}{1pt}}}

\makeatletter
\renewcommand\footnoterule{%
  \kern 15\p@
  \hrule \@width 2in \kern 2.6\p@
  \vspace{4pt}
}
\makeatother

\title{Diffusion Policy Improvement with Proposal-Conditioned Refinement Flows}
\author[1]{Junhyun~Ha}
\author[1]{Juho~Lee}
\author[1]{Byoungwoo~Park}

\affil[1]{KAIST}
\correspondingauthor{Correspondence to: \ccolor{\{junhyunha,~bw.park\}@kaist.ac.kr}}
\input{sections/1_abstract}

\begin{document}
\maketitle

\input{sections/2_introduction}
\input{sections/3_preliminaries}
\input{sections/4_method}
\input{sections/5_related_work}
\input{sections/6_experiments}
\FloatBarrier
\input{sections/7_conclusion}
\input{sections/8_statements}

\bibliographystyle{plainnat}
\bibliography{paper}

\clearpage
\appendix
\input{sections/9_appendix}

\end{document}

%% file: macro.tex
\definecolor{onlineblue}{rgb}{0,0.4609375,0.73046875}
\definecolor{offlinegray}{gray}{0.70}
\definecolor{cigray}{gray}{0.62}
\definecolor{arrowgray}{gray}{0.45}
\definecolor{oursbg}{RGB}{235,247,250}

\newcommand{\PReFlowbest}[1]{\textcolor{PReFlowBest}{\textbf{#1}}}
\newcommand{\PReFlowpair}[2]{%
  \makebox[1.50em][r]{\textcolor{PReFlowOffline}{#1}}%
  \makebox[0.82em][c]{\textcolor{PReFlowArrow}{$\scriptstyle\to$}}%
  \makebox[1.50em][l]{#2}%
}

\newcommand{\PReFlowpairCI}[4]{%
  \makecell[c]{%
    \makebox[\PReFlowResultWidth]{%
      \makebox[0.48\PReFlowResultWidth][c]{%
        {\fontsize{3.5}{4}\selectfont\textcolor{PReFlowOffline}{[#2]}}}%
      \hfill
      \makebox[0.48\PReFlowResultWidth][c]{%
        {\fontsize{3.5}{4}\selectfont\textcolor{PReFlowOffline}{[#4]}}}}%
    \\[-1pt]\PReFlowpair{#1}{#3}}%
}
\newcommand{\PReFlowgroup}[1]{%
  \addlinespace[3pt]
  \multicolumn{12}{@{}l@{}}{%
    \makebox[\linewidth][l]{%
      {\fontsize{7.1}{8.5}\selectfont\bfseries\scshape #1}%
      \enspace{\color{black!35}\hrulefill}}}\\[-1pt]
} 

\definecolor{onlineblue}{RGB}{0,94,158}
\definecolor{offlinegray}{gray}{0.43}
\definecolor{arrowgray}{gray}{0.58}
\definecolor{groupbg}{gray}{0.965}
\definecolor{allbg}{gray}{0.955}
\definecolor{oursbg}{RGB}{235,247,250}
\definecolor{oursallbg}{RGB}{220,240,245}

\definecolor{onlineblue}{RGB}{0,94,158}
\definecolor{offlinegray}{gray}{0.50}
\definecolor{arrowgray}{gray}{0.48}
\definecolor{allbg}{gray}{0.965}
\definecolor{oursbg}{RGB}{232,246,250}
\definecolor{oursallbg}{RGB}{216,239,245}

\input{math.tex}

\input{math_commands.tex}

\ifcsname lemma\endcsname\else

\fi

\crefname{assumption}{Assumption}{Assumptions}
\Crefname{assumption}{Assumption}{Assumptions}

\DeclareMathOperator{\sg}{sg}

\newcommand{\cA}{\mathcal{A}}
\newcommand{\dd}{\mathrm{d}}
\newcommand{\N}{\mathcal{N}}

\sethlcolor{background}
\DeclareRobustCommand{\best}[1]{\hl{#1}}

%% file: math.tex
\usepackage{xcolor,soul}
\usepackage{colortbl}

\newcommand{\calN}{{\mathcal{N}}}

\def\[#1\]{\begin{align}#1\end{align}}

\theoremstyle{plain}
\newtheorem{theorem}{Theorem}[section]
\newtheorem{proposition}[theorem]{Proposition}
\newtheorem{lemma}[theorem]{Lemma}

\def\[#1\]{\begin{align}#1\end{align}}

\definecolor{myellow}{RGB}{194, 125, 47}
\definecolor{mgreen}{RGB}{48, 160, 111}

\definecolor{mteal}{RGB}{221, 254, 242}
\definecolor{bgteal}{RGB}{236, 245, 245}
\definecolor{mpurple}{RGB}{120, 111, 177}
\definecolor{citationcolor}{RGB}{80, 90, 180}

\newcommand{\ccolor}[1]{{\color{citationcolor}{#1}}}

\usepackage{thmtools}
\usepackage{thm-restate}
\declaretheoremstyle[
  headfont=\bfseries,
  bodyfont=\normalfont,
  spaceabove=6pt, spacebelow=6pt,
  headpunct={.},
  postheadspace=1em
]{boldrestatable}

\definecolor{mygray}{gray}{0.95}

\newcommand{\graybox}[1]{%
  \begingroup
  \setlength{\fboxsep}{4pt}%
  \noindent\colorbox{mygray}{%
    \begin{minipage}{\dimexpr\linewidth-2\fboxsep\relax}%
      \setlength{\abovedisplayskip}{0pt}%
      \setlength{\belowdisplayskip}{0pt}%
      \setlength{\abovedisplayshortskip}{0pt}%
      \setlength{\belowdisplayshortskip}{0pt}%
      #1%
    \end{minipage}%
  }%
  \par
  \endgroup
}

\definecolor{mmgreen}{RGB}{243, 247, 243}
\definecolor{mmpurple}{RGB}{246, 242, 247}
\definecolor{mmccolor}{RGB}{235, 242, 250}

\definecolor{citationcolor}{RGB}{80, 90, 180}
\definecolor{background}{RGB}{240, 240, 250}
\definecolor{bggreen}{RGB}{229,242,229}
\newcommand{\cellbg}{\cellcolor{background}}

\DeclareRobustCommand{\best}[1]{\hl{#1}}

%% file: math_commands.tex
\usepackage{amsmath,amsfonts,bm}

\def\eqref#1{(\ref{#1})}

\def\1{\bm{1}}

\newcommand{\tr}{\mathrm{tr}}

\DeclareMathAlphabet{\mathsfit}{\encodingdefault}{\sfdefault}{m}{sl}
\SetMathAlphabet{\mathsfit}{bold}{\encodingdefault}{\sfdefault}{bx}{n}

\newcommand{\E}{\mathbb{E}}



%% file: sections/1_abstract.tex
\begin{abstract}
Diffusion and flow policies can model complex behaviors in offline reinforcement learning (RL). 
However, penalizing their KL divergence from the behavior policy can discourage actions having high critic values with low behavior density. 
Directly refining behavior proposals may be an alternative, yet Gaussian or deterministic editors limit expressiveness to represent multiple separated modes for the same proposal. 
In this work, we introduce Proposal-Conditioned Refinement Flows
(PReFlow), a policy extraction method combining critic-based proposal selection with a conditional refinement flow. 
To optimize proposal selection and refinement
together, we formulate a KL-regularized objective whose optimum induces a Gibbs policy over final actions under a Gaussian-smoothed behavior prior. 
The refinement flow can represent multiple high value modes, while a proposal-centered Gaussian reference regulates large action changes. This Gaussian reference further enables us to make use of simulation-free, closed form adjoint matching targets from sampled endpoints and critic gradients, yielding a single velocity regression loss without
a backward adjoint solve. 
On 50 OGBench tasks, PReFlow achieves competitive offline performance and the highest aggregate score among the compared methods after online fine-tuning, reaching 91\% after 500K environment steps.
\end{abstract}

%% file: sections/2_introduction.tex
\section{Introduction}

Diffusion and flow models provide expressive policy classes for offline reinforcement learning (RL), capable of representing complex, multimodal action distributions~\citep{diffusionql,idql,fql}. The central challenge is to improve such policies using a learned critic while avoiding out-of-distribution actions, where value estimates can be unreliable~\citep{bcq,cql}. A common approach is behavior regularized policy improvement, which balances critic value against deviation from the behavior policy. QAM~\citep{qam} optimizes an expressive flow policy under KL regularization. However, this can restrict improvement when high value actions have low behavior density, even if they lie close to likely behavior actions.

Candidate-based policy extraction methods offer another approach, using the critic to select promising actions sampled from the behavior policy \citep{emaq,idql}. But their output is restricted to the candidates, even when better actions lie nearby. This suggests treating the selected action as a proposal that can be further refined toward higher-value actions.

Residual policies learn additive refinements to base actions~\citep{rpl,residualrl,policydecorator}, where multiple distinct refinements may improve a given state and proposal.    
However, existing residual policies based on Gaussian~\citep{expo,qam} or deterministic~\citep{deflow,fidec} refinements offer limited flexibility.
SPAR~\citep{spar} supports multimodal refinements through a conditional VAE,
but requires residual reconstruction and latent
self-imitation.
This leads to the following question:

\graybox{%
  \begin{center}
  \textit{How can we make local action refinement both expressive and efficient to learn?}
  \end{center}
}

To answer this question, we introduce
\textbf{Proposal-Conditioned Refinement Flows (PReFlow)},
a policy extraction method that models additive action refinements
with a conditional flow.
A behavior flow generates proposals, and the critic selects
the one with the highest value.
Conditioned on the state and selected proposal, the flow samples a refinement, which is
added to the proposal to produce the final action
(Figure~\ref{fig:method_result}).
We regularize the refined-action distribution toward a Gaussian
centered at the proposal.
This reference discourages large action changes without restricting
the refinement distribution to a Gaussian family, allowing distinct
high value refinements of the same proposal.

\begin{figure}[!t]
    \centering
    \includegraphics[width=0.98\linewidth]
    {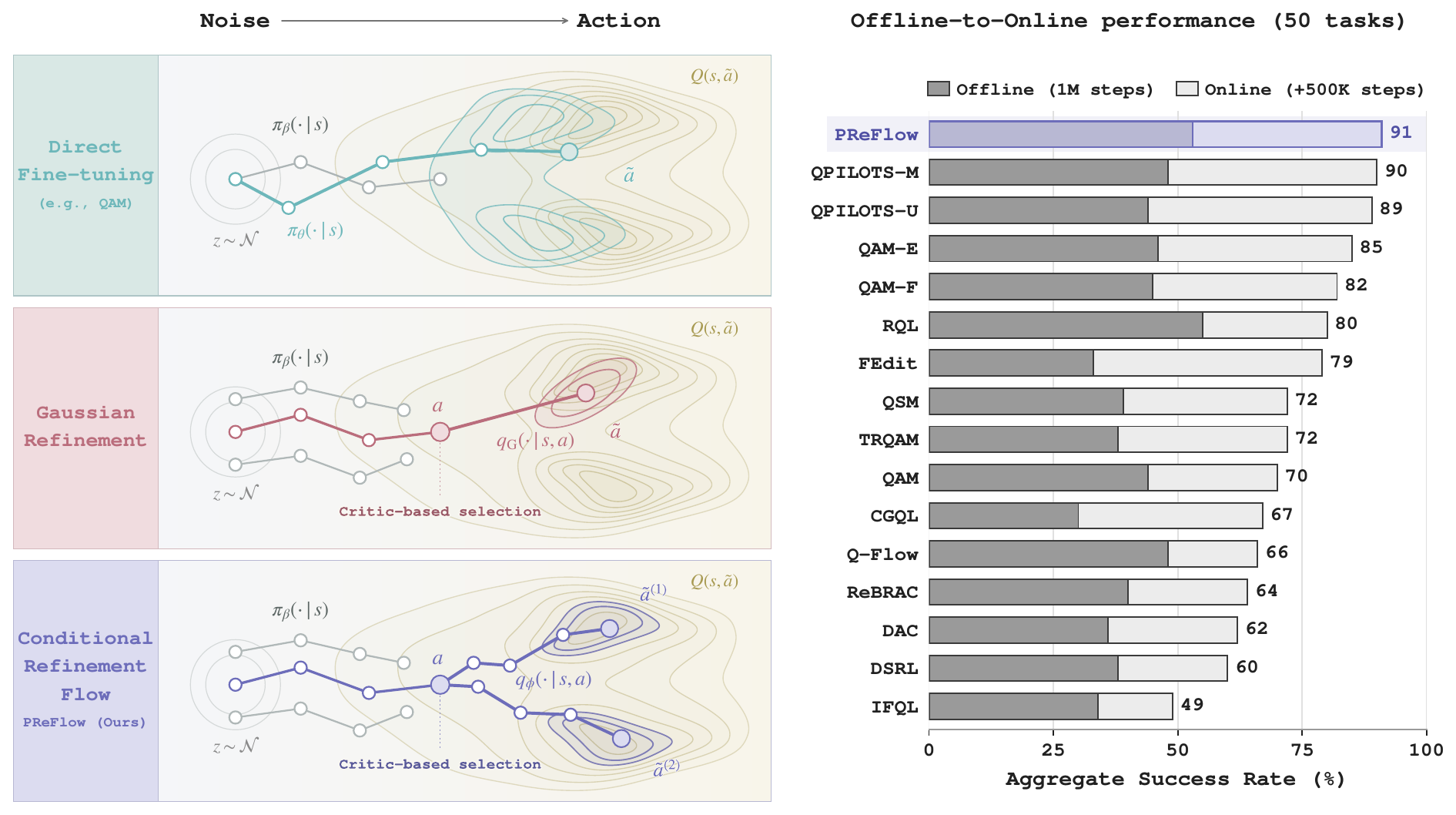}
    \caption{
\textbf{PReFlow: proposal selection and stochastic refinement.}
A behavior flow generates proposals, the critic selects one, and a
conditional flow refines it to produce final action $\tilde{a}=a+\Delta a$.
The performance panel reports mean success rates across the 50 OGBench tasks
(Table~\ref{tab:main_combined}).
}
    \label{fig:method_result} 

\end{figure} 

We formulate a joint KL regularized objective for proposal selection and conditional refinement,
whose optimum induces a Gibbs policy under a Gaussian-smoothed behavior prior. The Gaussian
refinement reference also enables closed form adjoint matching targets~\citep{eam, adjointsampling}. Given refinement endpoints from forward rollouts of the current flow, we construct these
targets without further simulation and train the conditional flow with a single velocity regression loss.

We evaluate PReFlow on 50 OGBench tasks~\citep{ogbench}, comparing it
against a broad range of policy extraction methods for flow and
diffusion policies. PReFlow achieves the highest aggregate
offline-to-online performance among the compared methods while
remaining competitive in offline RL. Our controlled ablations
further show gains from expressive flow refinement over Gaussian
editing, improvements even with a single proposal, and additional
benefits from proposal selection.


%% file: sections/3_preliminaries.tex
\section{Preliminaries}
\label{sec:prelim}

We introduce the problem setting and the background underlying our conditional refinement method. 

\begin{sloppypar}
\paragraph{Behavior regularized policy extraction.}
We consider a Markov decision process (MDP)
$\mathcal M=(\mathcal S,\mathcal A,P,\gamma,R,\mu)$, where $\mathcal S$ is
the state space, $\mathcal A=\mathbb R^{A}$ is the
action space, $P:\mathcal S\times\mathcal A\to\Delta\mathcal S$ is the
transition function, $\gamma\in[0,1)$ is the discount factor,
$R:\mathcal S\times\mathbb R^{A}\to\mathbb R$ is the reward function, and
$\mu\in\Delta\mathcal S$ is the initial state distribution. The goal of Offline RL is to learn a policy $\pi_\theta:\mathcal S\to\Delta\mathcal A$
that maximizes the expected discounted return
$\mathbb E
\bigl[\sum_{k=0}^{\infty}\gamma^kR(s_k,a_k)\bigr]$
from a dataset $\mathcal D=\{(s_i,a_i,s_i',r_i)\}_{i=1}^{|\mathcal D|}$
collected by behavior policy $\pi_\beta$.
Offline-to-Online (O2O) RL further aims to fine-tune this policy with limited online
interactions.
Given a learnable critic $Q_\psi:\mathcal S\times\mathcal A\to\mathbb R$,
we formulate policy extraction with KL regularization
toward the behavior policy
\begin{equation}
    \max_{\pi}\;
    \mathbb E_{\tilde a\sim\pi(\cdot\mid s)}\bigl[Q_\psi(s,\tilde a)\bigr]
    -\frac{1}{\tau(s)}\,\operatorname{KL}(\pi\,\|\,\pi_\beta),
    \label{eq:behavior regularized-objective}
\end{equation}
where $\tau(s)>0$ controls the critic value relative to the KL penalty~\citep{awac}. Optimizing over action distributions gives the Gibbs policy, which favors actions with higher critic values
\begin{equation}
    \pi^\star(\tilde a\mid s)\propto
    \pi_\beta(\tilde a\mid s)\,e^{\tau(s)Q_\psi(s,\tilde a)}.
    \label{eq:gibbs-policy}
\end{equation}
\end{sloppypar}

\paragraph{Flow matching.}
To learn the reference $\pi_\beta$ in~\eqref{eq:behavior regularized-objective}, we fit a velocity field $f_\beta$ to $\mathcal D$. This velocity field transforms Gaussian noise into a behavior action through an ordinary differential equation (ODE), 
\begin{equation}
\mathrm d X_t
=f_\beta(X_t,t\mid s) \mathrm d t, \quad X_0 \sim \mathcal{N}(0, I_{A}), \quad X_1\sim\pi_\beta(\cdot \mid s).
\label{eq:fm-ode}
\end{equation}
To train $f_\beta$, we pair a dataset action $a$ at state $s$ with Gaussian noise $x_0$ and form the linear interpolation $x_t=(1-t)x_0+ta$. Its velocity provides the regression target in the flow matching loss~\citep{lipman2023flow,liu2023flow,albergo2023building} 
\begin{equation}
\mathcal L_{\mathrm{FM}}(\beta)
=
\mathbb E_{t\sim\mathcal U(0,1),\, x_0\sim\mathcal N,\,
(s,a)\sim \mathcal D
}
\left[
\left\|
f_\beta(x_t,t \mid s)-(a-x_0)
\right\|_2^2
\right],
\label{eq:flow-matching}
\end{equation}

\paragraph{Q-learning with adjoint matching.}
QAM~\citep{qam} learns a flow $f_\theta$
to approximate the Gibbs policy~\eqref{eq:gibbs-policy}.
Using $f_\theta$ as the ODE velocity in~\eqref{eq:fm-ode}
defines the policy $\pi_\theta$ and final action $\tilde a$,
\begin{equation}
    \tilde a=\mathrm{ODE}_{f_\theta}(\varepsilon;s)
    \sim\pi_\theta(\cdot\mid s),
    \quad \varepsilon\sim\mathcal N(0,I_A).
\end{equation}
Here $\mathrm{ODE}_{f}(\varepsilon;s)$ denotes the endpoint
obtained by integrating the flow $f$ from $\varepsilon$.
To train $f_\theta$, QAM uses adjoint matching (AM)~\citep{am},
with $f_\beta$ fixed as the reference during each policy update.
To construct regression targets for this update,
AM samples trajectories $\mathbf X=\{X_t\}_{t\in[0,1]}$
from a stochastic differential equation (SDE).
For a fixed state $s$, the SDE uses the current flow $f_\theta$
and the memoryless noise schedule $\sigma_t=\sqrt{2(1-t)/t}$,
\begin{equation}
    \mathrm dX_t=\Bigl(2f_\theta(X_t,t\mid s)-X_t/t\Bigr)\mathrm dt
    +\sigma_t\,\mathrm dB_t,
    \qquad X_0\sim\mathcal N(0,I_A),
    \label{eq:memoryless-sde}
\end{equation}
where $B_t$ is standard Brownian motion.
Along each trajectory, AM computes the lean adjoint
$\lambda(\mathbf X,t)$
by propagating the endpoint critic gradient backward
using the Jacobian of the reference drift,
\begin{equation}
    \mathrm d\lambda(\mathbf X,t)
    =-\nabla_{X_t}\left[2f_\beta(X_t,t\mid s)-X_t/t\right]\lambda(\mathbf X,t)\,\mathrm dt,
    \quad
    \lambda(\mathbf X,1)=-\tau(s)\nabla_{X_1}Q_\psi(s,X_1).
    \label{eq:lean-adjoint}
\end{equation}
The lean adjoint defines the velocity target
$f_\beta-\sigma_t^2\lambda/2$ for $f_\theta$,
yielding the regression loss
\begin{equation}
    \mathcal L_{\mathrm{AM}}(\theta)
    =\mathbb E_{\mathbf X}\Bigl[\int_0^1
    \Bigl\|2\bigl(f_\theta(X_t,t\mid s)-f_\beta(X_t,t\mid s)\bigr)/\sigma_t
    +\sigma_t\lambda(\mathbf X,t)\Bigr\|_2^2\,\mathrm dt\Bigr].
    \label{eq:adjoint matching}
\end{equation}
Because~\eqref{eq:lean-adjoint} depends on the Jacobian of $f_\beta$,
constructing these targets involves a sequential backward pass
of vector--Jacobian products along the entire sampled SDE trajectory.

%% file: sections/4_method.tex
\section{Policy Improvement with Conditional Refinement Flows}
\label{sec:method}

In this section, we introduce PReFlow, which learns local value-guided refinements conditioned on proposals from the behavior policy. 
We motivate this formulation and develop a joint regularized
objective for proposal selection and refinement.
From this objective, we derive simulation-free adjoint matching
targets to train the flow.
Proofs and technical details are provided in
Appendix~\ref{app:proofs}.

\subsection{Towards Expressive Conditional Refinement}
\label{sec:refinement}

\paragraph{Behavior regularized flow policies.}
The goal of policy extraction~\eqref{eq:behavior regularized-objective}
is to generate final actions from
the Gibbs policy $\pi^\star$ in~\eqref{eq:gibbs-policy}.
QAM~\citep{qam} approximates this target by optimizing $f_\theta$
with adjoint matching~\eqref{eq:adjoint matching},
so the learned policy directly generates
$\tilde a\sim\pi_\theta(\cdot\mid s)$.
The KL penalty introduces the behavior-density factor
$\pi_\beta(\tilde a\mid s)$ into the Gibbs target.
Low behavior density can therefore reduce the weight of high value actions
even when they lie close to behavior actions.
This motivates a conditional approach
that regularizes action changes relative to each behavior proposal.

\paragraph{Existing action editors.}
At state $s$, the behavior flow $f_\beta$
trained with the flow matching loss~\eqref{eq:flow-matching}
generates a proposal $a$
by integrating the ODE~\eqref{eq:fm-ode}.
With $a$ fixed, an action editor learns the conditional distribution
$\pi^{\mathrm{ref}}(\cdot \mid s,a)$
of an additive refinement $\Delta a$.
The final action is generated in two stages,
\begin{equation}
    a\sim\pi_\beta(\cdot\mid s),
    \quad
    \Delta a\sim\pi^{\mathrm{ref}}(\cdot\mid s,a),
    \quad
    \tilde a=a+\Delta a.
\end{equation}
Because multiple distinct refinements of the same proposal can yield high critic values at a given state, the conditional policy can benefit from representing multiple modes.

For example, Gaussian-based editors have been proposed to model stochastic refinements~\citep{expo}, while deterministic editors have been developed to map each state and proposal $(s,a)$ to a single refinement~\citep{deflow, fidec}. 
Both parameterizations might restrict the refinement distribution $\pi^{\mathrm{ref}}(\cdot | s, a)$ to be a single mode, limiting its
flexibility in capturing distinct refinement modes. A multimodal refinement policy based on a conditional VAE has been proposed to address this limitation~\citep{spar}, but its training combines weighted residual reconstruction with an additional latent self-imitation objective. To retain this expressiveness while simplifying training, we seek a conditional refinement policy learned with a single regression objective.

\paragraph{Proposal-conditioned refinement flows.}
To this end, we introduce PReFlow, which refines a behavior proposal $a \sim \pi_\beta(\cdot|s)$ using an expressive conditional flow. We train the refinement flow policy $\pi^{\mathrm{ref}}(\cdot|s,a)$
to maximize the expected critic value of the refined action $\tilde a$, with KL regularization toward the Gaussian reference $\rho_\delta=\calN(0,\delta^2 I)$,
where $\delta>0$ sets the reference refinement scale.
Equivalently, the distribution of $\tilde a=a+\Delta a$
is regularized toward $\calN(a,\delta^2 I)$,
therefore the proposal $a$ serves as both
an input to the conditional flow
and the center of the regularization reference.
This reference discourages large action changes
without restricting $\pi^{\mathrm{ref}}$ to a Gaussian family,
allowing multiple high value refinements of the same proposal as depicted in Fig.~\ref{fig:toy_qvpo}.

\begin{figure}[!t]
    \centering
    \includegraphics[width=0.98\linewidth]
    {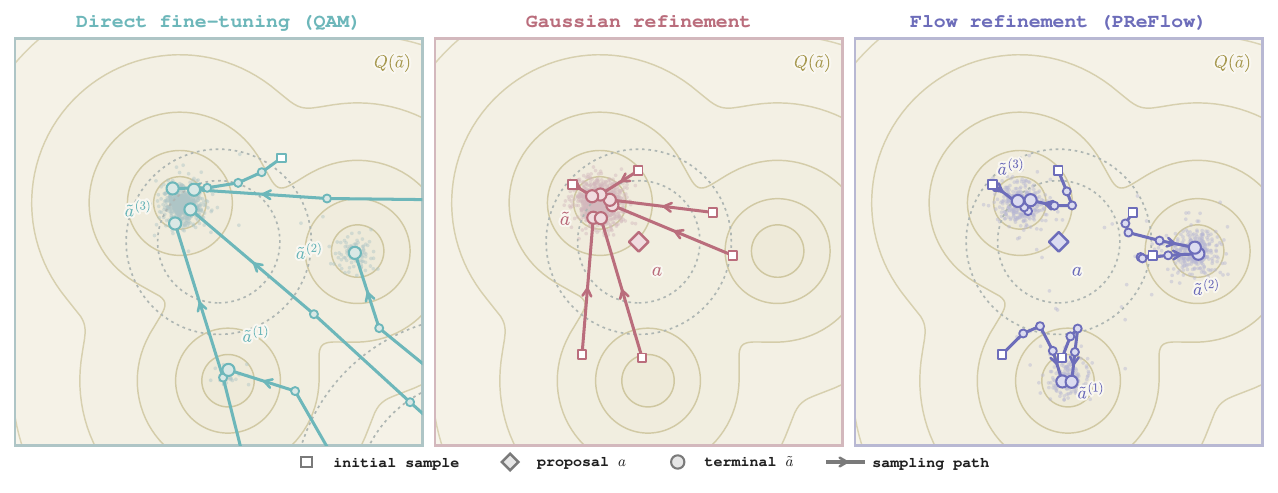}
    \caption{
\textbf{Toy comparison of policy extraction methods.}
Fine-tuning the full flow (QAM) undercovers the two modes farther from
the proposal, and Gaussian refinement concentrates on a single mode.
PReFlow captures all three modes. Appendix~\ref{app:toy-edit-family}
provides details of the toy experiment.
}
    \label{fig:toy_qvpo} 

\end{figure}

The Gaussian reference also simplifies training
by admitting a reference process with closed form
intermediate-state distributions and lean adjoints~\citep{am, eam}.
Given refinement endpoints sampled from the current flow,
these expressions enable simulation-free construction
of regression targets from Gaussian bridge samples
and endpoint critic gradients.
The resulting single regression objective requires
neither a backward adjoint solve
nor an additional residual-reconstruction loss.
We formalize the conditional objective
and use its optimal value to guide proposal selection
in the next section,
and then derive the regression objective
in Section~\ref{sec:am}.

\subsection{A Regularized Objective for Conditional Policy Improvement}
\label{sec:lift}

Starting from policy extraction objective~\eqref{eq:behavior regularized-objective}, we first derive a refinement objective for $\Delta a$ conditioned on a fixed behavior proposal $a$ and then use its optimal value to guide proposal selection.

\paragraph{Conditional refinement objective.}
For state $s$, we sample a proposal $a\sim\pi_\beta(\cdot\mid s)$ from the behavior flow and hold it fixed during refinement.
We optimize $\pi_a(\cdot\mid s)$, the conditional distribution of the refined action $\tilde a=a+\Delta a$.
To regularize action changes relative to $a$, we replace the behavior reference in~\eqref{eq:behavior regularized-objective} with $\mathcal N(a,\delta^2I)$, as motivated in Section~\ref{sec:refinement}.
For $\tau=\tau(s)>0$, the objective is
\begin{equation}
    \mathbb E_{\tilde a\sim\pi_a(\cdot\mid s)}
    [Q_\psi(s,\tilde a)]
    -\frac{1}{\tau}
    \operatorname{KL}\!\left(
        \pi_a(\cdot\mid s)
        \,\|\,\mathcal N(a,\delta^2I)
    \right).
    \label{eq:behavior-proposal}
\end{equation}
With proposal $a$ fixed, writing $\Delta a=\tilde a-a$ gives $\pi_a(a+\Delta a\mid s)=\pi^{\mathrm{ref}}(\Delta a\mid s,a)$.
Since the reference for $\tilde a$ is $\mathcal N(a,\delta^2I)$, the corresponding reference for $\Delta a$ is $\rho_\delta=\mathcal N(0,\delta^2I)$.
By the translation invariance of KL divergence, we can rewrite the refinement objective in~\eqref{eq:behavior-proposal} as
\begin{equation}
    \mathcal J_{s,a}(\pi^{\mathrm{ref}})
    =
    \mathbb E_{\Delta a\sim\pi^{\mathrm{ref}}(\cdot\mid s,a)}
    [Q_\psi(s,a+\Delta a)]
    -
    \tfrac{1}{\tau}
    \operatorname{KL}\!\left(
        \pi^{\mathrm{ref}}(\cdot\mid s,a)
        \,\|\,\rho_\delta
    \right).
    \label{eq:conditional-objective}
\end{equation}

Applying the Gibbs solution~\eqref{eq:gibbs-policy} to~\eqref{eq:conditional-objective} gives
\begin{equation}
    \pi^{\mathrm{ref},\star}(\Delta a\mid s,a)
    \propto
    \exp\!\left(
        \tau Q_\psi(s,a+\Delta a)
        -\frac{\|\Delta a\|^2}{2\delta^2}
    \right).
    \label{eq:refinement-target}
\end{equation}
The target balances critic value against refinement magnitude
without directly penalizing low behavior density.
It can be multimodal despite its Gaussian reference
when the critic favors distinct refinements.

To provide intuition for this optimality, we further define $Z(s,a)=\mathbb E_{\Delta a\sim\rho_\delta}[e^{\tau Q_\psi(s,a+\Delta a)}]$, the partition function of the optimal refinement policy.
Then, we can rewrite the objective~\eqref{eq:conditional-objective} as
\begin{equation}
    \mathcal J_{s,a}(\pi^{\mathrm{ref}})
    =\frac{1}{\tau}\log Z(s,a)
    -\frac{1}{\tau}\operatorname{KL}\!\left(
        \pi^{\mathrm{ref}}\,\|\,\pi^{\mathrm{ref},\star}
    \right).
\end{equation}
The first term is independent of $\pi^{\mathrm{ref}}$, and the KL divergence reaches its minimum at $\pi^{\mathrm{ref}}=\pi^{\mathrm{ref},\star}$.

\paragraph{Joint refinement and proposal selection.}
We have optimized $\pi^{\mathrm{ref}}$ for a fixed proposal $a$.
Since the optimal value of~\eqref{eq:conditional-objective} depends on the proposal
and balances critic value against the KL penalty,
we now optimize proposal selection with refinement.
We denote this optimal conditional value by
\begin{equation}
    V_\delta(s,a)
    =
    \sup_{\pi^{\mathrm{ref}}}
    \mathcal J_{s,a}(\pi^{\mathrm{ref}})
    =\frac{1}{\tau}\log Z(s,a).
    \label{eq:refinement-value}
\end{equation}
Let $\pi^{\mathrm{base}}(\cdot\mid s)$ denote the distribution of selected proposals.
To evaluate selection and refinement jointly, we average $\mathcal J_{s,a}(\pi^{\mathrm{ref}})$ over $a\sim\pi^{\mathrm{base}}(\cdot\mid s)$.
Since proposals come from $\pi_\beta$, we regularize $\pi^{\mathrm{base}}$ toward it via KL penalty.
Using the same coefficient $1/\tau$ for both stages gives the joint objective
\begin{equation}
    \mathcal J_s(\pi^{\mathrm{base}},\pi^{\mathrm{ref}})
    =
    \mathbb E_{a\sim\pi^{\mathrm{base}}(\cdot\mid s)}
    [\mathcal J_{s,a}(\pi^{\mathrm{ref}})]
    -
    \tfrac{1}{\tau}
    \operatorname{KL}\!\left(
        \pi^{\mathrm{base}}(\cdot\mid s)
        \,\|\,\pi_\beta(\cdot\mid s)
    \right).
    \label{eq:joint-objective}
\end{equation}

With optimal conditional refinement, the first term becomes $\mathbb E_{a\sim\pi^{\mathrm{base}}}[V_\delta(s,a)]$.
Applying the same Gibbs solution with $V_\delta$ in place of $Q_\psi$ therefore gives the optimal proposal-selection policy
\begin{equation}
    \pi^{\mathrm{base},\star}(a\mid s)
    \propto
    \pi_\beta(a\mid s)\exp\!\left(\tau V_\delta(s,a)\right)
    =\pi_\beta(a\mid s)Z(s,a).
    \label{eq:proposal-target}
\end{equation}
This policy favors proposals according to their value after regularized refinement.

\paragraph{Optimal two-stage policy.}
We now relate this two-stage policy to the final action target in~\eqref{eq:gibbs-policy}.
Using the common coefficient $1/\tau$, the KL chain rule combines the two penalties into a joint KL penalty relative to $\pi_\beta(a\mid s)\rho_\delta(\Delta a)$.
This reference first draws a behavior proposal and then adds an independent Gaussian refinement.
Its final action distribution is therefore
$\bar\pi_{\beta,\delta}(\tilde a\mid s)
    =\int\pi_\beta(a\mid s)\rho_\delta(\tilde a-a)\,\mathrm da.$
This Gaussian-smoothed behavior prior averages Gaussian neighborhoods centered at behavior proposals.
The joint optimum reweights the reference by $e^{\tau Q_\psi(s,\tilde a)}$.
Since this weight depends only on the final action, integrating out the proposal gives the following result.

\begin{proposition}[Final policy under proposal-relative refinement]
\label{prop:lift}
Jointly optimizing~\eqref{eq:joint-objective} over all proposal selection and conditional refinement policies yields the final action policy

\graybox{
\begin{equation}
    \pi_\delta^\star(\tilde a\mid s)
    \propto
    \bar\pi_{\beta,\delta}(\tilde a\mid s)e^{\tau Q_\psi(s,\tilde a)}.
    \label{eq:final-policy}
\end{equation}
}
\end{proposition}

The two stages of proposal selection and refinement recover the Gibbs policy improvement rule with $\bar\pi_{\beta,\delta}$ replacing $\pi_\beta$.
Proposals originate from $\pi_\beta$, while $\delta$ controls the spread of their Gaussian neighborhoods.
Nearby proposals can provide reference density to actions with low behavior density.

\par
\begingroup
\setlength{\columnsep}{14pt}
\setlength{\intextsep}{0.9pt}
\begin{wrapfloat}{algorithm}{r}{0.4\textwidth}
\captionsetup{font=small,labelfont=bf,labelsep=space,
        singlelinecheck=false,skip=3pt,position=top}
    \caption{PReFlow: Action generation}
    \label{algo:preflow-action}
    \begin{algorithmic}
    \setlength{\leftskip}{-\leftmargin}

    \Statex \textbf{Input:}
    $s$, $f_\beta$, $v_\phi$, $Q_\psi$, $K$, $\delta$.

    \State $\varepsilon'_{1:K}\overset{\mathrm{i.i.d.}}{\sim}
        \mathcal N(0,I)$
    \State $a_{1:K}\gets
        \mathrm{ODE}_{f_\beta}(\varepsilon'_{1:K};s)$
    \Comment{\eqref{eq:fm-ode}}

    \State $I\gets\arg\max_{i\in\{1,\ldots,K\}}Q_\psi(s,a_i)$
    \State $a\gets a_I$
    \State $y_0\sim\mathcal N(0,I)$
    \State $y_1\gets\mathrm{ODE}_{v_\phi}(y_0;s,a)$
    \Comment{\eqref{eq:refinement-policy}}

    \Statex \textbf{Output:} $(a,y_1)$.
    \end{algorithmic}
\end{wrapfloat}

\paragraph{Practical proposal selection.}
Computing $V_\delta(s,a)$ requires integrating over possible
refinements, whereas $Q_\psi(s,a)$ is available from a critic
evaluation.
For a sufficiently smooth critic, the small-$\delta$
expansion $V_\delta(s,a)=Q_\psi(s,a)+O(\delta^2)$
motivates using the critic value as an approximation.
For $K$ candidates $a_{1:K}\sim\pi_\beta(\cdot\mid s)$,
replacing $V_\delta(s,a_i)$ by $Q_\psi(s,a_i)$
in~\eqref{eq:proposal-target} gives finite-temperature
selection weights
$w_i\propto\exp\!\left(\tau Q_\psi(s,a_i)\right)$.
We instead use greedy selection as shown in Algorithm~\ref{algo:preflow-action}; 
Appendix~\ref{app:abl-selection} provides comparison of two methods.

\par
\endgroup

\subsection{Simulation-free Training of Flow Refinement}
\label{sec:am} 

\paragraph{Conditional flow parameterization.}
To parameterize the refinement policy, we write $\Delta a=\delta y$
and $\tilde a=a+\delta y$.
Using the Gaussian reference $\rho=\mathcal N(0,I)$ for $y$,
we can rewrite~\eqref{eq:refinement-target} as 
\begin{equation}
    q^\star(y\mid s,a)
    \propto
    \rho(y)\exp\!\left(\tau Q_\psi(s,a+\delta y)\right).
    \label{eq:conditional-target}
\end{equation}
We approximate this target with a conditional flow:
\begin{equation}
    \frac{\mathrm dY_t}{\mathrm dt}
    =v_\phi(Y_t,t\mid s,a),
    \qquad Y_0\sim\mathcal N(0,I),
    \qquad Y_1\sim q_\phi(\cdot\mid s,a).
    \label{eq:refinement-policy}
\end{equation}
With the state and proposal fixed throughout the rollout,
the endpoint yields the refinement $\Delta a=\delta Y_1$.

\par
\begingroup
\setlength{\columnsep}{14pt}
\setlength{\intextsep}{0.9pt}
\begin{wrapfloat}{algorithm}{r}{0.4\textwidth}
    \captionsetup{font=small,labelfont=bf,labelsep=space,
        singlelinecheck=false,skip=3pt,position=top}
    \caption{PReFlow: Flow Update}
    \label{algo:PReFlow}
    \begin{algorithmic}
    \setlength{\leftskip}{-\leftmargin}

    \Statex \textbf{Input:}
    $\mathcal D$, $f_\beta$, $v_\phi$, $Q_\psi$, $\delta$, $\tau$.

    \Statex \textbf{Generate endpoints}
    \State Sample a replay state $s\sim\mathcal D$.
    \State $(a,y_1)\gets$ Algorithm~\ref{algo:preflow-action}, $K=1$. 
    \State $(a,y_1)\gets\operatorname{sg}[(a,y_1)]$
    \Comment{{\scriptsize $\operatorname{sg}$: stop-gradient}}

    \Statex \vspace{3pt}\textbf{Construct the velocity target}
    \State $t\sim\mathcal U(0,1)$, $\varepsilon\sim\mathcal N(0,I)$
    \State Compute $y_t$ using~\eqref{eq:reference-bridge}.
    \State $\tilde a_1\gets a+\delta y_1$
    \State Compute $v^{\mathrm{tgt}}$ using~\eqref{eq:vtgt}.

    \Statex \vspace{3pt}\textbf{Update the refinement flow}
    \State Update $\phi$ by minimizing~\eqref{eq:actor-loss}.

    \Statex \vspace{2pt}\textbf{Output:} $v_\phi$.
    \end{algorithmic} 
\end{wrapfloat}

\paragraph{Simulation-free target construction.}
To train this flow, we construct regression targets using the stochastic optimal
control (SOC) formulation of Adjoint Matching~\citep{am}.
QAM~\citep{qam} uses the pretrained behavior flow as its reference,
so constructing targets requires the sequential backward solve
of~\eqref{eq:lean-adjoint}.
Our objective instead regularizes the refinement toward the Gaussian
reference $\rho$, while the behavior flow only generates proposals.
We therefore use the memoryless linear Gaussian reference flow
~\citep{eam}, whose velocity field is
$v_0(y,t)=(2t-1)y/d(t)$ with $d(t)=t^2+(1-t)^2$.
Using $v_0$ in place of $f_\theta$ in the memoryless
SDE~\eqref{eq:memoryless-sde} defines a Gaussian reference process
for the normalized refinement $y$, with terminal distribution $\rho$
and closed form bridges for constructing the targets.

Since the conditional target is independent of the
proposal-selection rule, we draw proposals directly from $\pi_\beta$.
Following Adjoint Sampling~\citep{adjointsampling}, we generate
an endpoint $y_1\sim q_\phi(\cdot\mid s,a)$ with a forward flow
rollout and define the refined action as
$\tilde a_1=a+\delta y_1$.
We then draw $t\sim\mathcal U(0,1)$ and directly sample
the intermediate state from the reference process
conditioned on $y_1$ as
\begin{equation}
    y_t=ty_1+(1-t)\varepsilon,
    \qquad \varepsilon\sim\mathcal N(0,I).
    \label{eq:reference-bridge}
\end{equation}
\par
\endgroup

\begin{proposition}[Simulation-free refinement velocity target]
\label{prop:am}
For an endpoint $y_1$ and a bridge sample $y_t$ from
\eqref{eq:reference-bridge}, the adjoint matching velocity target
for the Gaussian reference above is 

\graybox{
\begin{equation}
    v^{\mathrm{tgt}}(y_t,y_1,t\mid s,a)
    =\frac{2t-1}{d(t)}y_t
    +\frac{\tau\delta(1-t)}{d(t)}
        \nabla_{\tilde a}Q_\psi(s,\tilde a_1).
    \label{eq:vtgt}
\end{equation}
}
\end{proposition}

With $v_\phi$ and $v_0$ replacing $f_\theta$ and $f_\beta$,
the integrand in~\eqref{eq:adjoint matching} becomes
$\frac{4}{\sigma_t^2}\|v_\phi-v^{\mathrm{tgt}}\|_2^2$.
Following RAM~\citep{ram}, we omit this time weight
and use unweighted regression on the bridge samples,
holding sampled endpoints and targets fixed during the update.
\begin{equation}
    \mathcal L_{\mathrm{actor}}(\phi)
    =
    \mathbb E_{\substack{
        s\sim\mathcal D,\;
        a\sim\pi_\beta(\cdot\mid s),\;
        y_1\sim q_\phi(\cdot\mid s,a)\\
        t\sim\mathcal U(0,1),\;
        \varepsilon\sim\mathcal N(0,I)
    }}\!\Bigl[
        \bigl\|v_\phi(y_t,t\mid s,a)
        -\operatorname{sg}[v^{\mathrm{tgt}}]\bigr\|_2^2
    \Bigr],
    \label{eq:actor-loss}
\end{equation}
Algorithm~\ref{algo:PReFlow} summarizes the refinement flow update.
The full training algorithm, including critic and behavior flow
updates, is provided in Algorithm~\ref{algo:preflow-full} in Appendix~\ref{app:training-procedure}.

%% file: sections/5_related_work.tex
\section{Related Work}

\paragraph{Policy extraction with diffusion and flow policies.}
Diffusion and flow policies incorporate critic information through
end-to-end gradients~\citep{diffusionql}, value-weighted
fitting~\citep{efm}, score or noise-regression
objectives~\citep{qsm,dac}, and energy guidance~\citep{qgpo,flowq}.
FQL~\citep{fql} regularizes a separate noise-conditioned one-step
actor toward a flow behavior model, while Q-Flow and
RQL~\citep{qflow,rql} learn values over intermediate generative states.
QPILOTS and QGF apply critic-gradient guidance during
sampling~\citep{qpilots,qgf}. PReFlow instead trains its conditional
refinement flow with endpoint critic gradients.
At inference, critic values guide selection among behavior proposals,
while refinement requires no critic-gradient evaluations.

\paragraph{Behavior proposals and conditional action refinement.}
Selection among behavior proposals~\citep{emaq,idql,dpe}
can be combined with local improvement~\citep{bcq,parl,expo}.
Residual policies learn additive modifications to base-policy
actions~\citep{rpl,residualrl,policydecorator}.
EXPO and QAM-E use Gaussian editors~\citep{expo,qam}, whereas
DeFlow and Fisher Decorator learn deterministic
refinements~\citep{deflow,fidec}.
SPAR~\citep{spar} learns multimodal residuals through a conditional
VAE and latent self-imitation, and evaluates flow-based variants.
FLAG~\citep{flag} uses local Gaussian policies around flow-generated
anchors to supervise policy improvement. PReFlow trains a conditional refinement flow with adjoint matching under proposal-centered KL regularization.



%% file: sections/6_experiments.tex
\section{Experiments}
\label{sec:exp}

We evaluate PReFlow in offline and offline-to-online reinforcement
learning and examine the effects of refinement policy design,
along with computational cost and hyperparameter sensitivity. Detailed experimental settings and hyperparameter settings are provided
in Appendix~\ref{app:experiments}.

\paragraph{Domains and datasets.}
We evaluate PReFlow on a total of 50 tasks from the OGBench single-task
suite~\citep{ogbench}, spanning antmaze, humanoidmaze, scene,
puzzle, and cube domains.
We follow QAM's advanced ten domain variants and reward configurations~\citep{qam}.

\paragraph{Baselines.}
We compare with the policy extraction baselines originally evaluated by
\citet{qam}: Gaussian policy (ReBRAC~\citep{rebrac}),
critic-guided method (QSM~\citep{qsm}),
post-processing methods (DSRL, FEdit, and IFQL~\citep{dsrl,qam}),
and adjoint matching methods (QAM, QAM-F, and QAM-E~\citep{qam}).
We also compare with more recent policy extraction methods for flow policies,
QPILOTS-M/U~\citep{qpilots}, Q-Flow~\citep{qflow}, RQL~\citep{rql}, and
TRQAM~\citep{dong2026trustregionqadjoint}.

\subsection{Offline and Offline-to-Online Performance}
\label{sec:exp-performance}

Following the QAM protocol~\citep{qam}, we train for 1M offline
gradient steps and fine-tune for 500K online environment steps.
We report success rates at the end of each phase. Domain scores
average five tasks, and the overall score averages all 50 tasks.
We use eight seeds per task for PReFlow and the reproduced baselines.
Task-level results and learning curves are in
Appendix~\ref{app:full-results}.

\begin{table}[!t]
\centering
\begingroup
\def\PReFlowbest#1{{%
  \setlength{\fboxsep}{0.6pt}%
  \colorbox{background}{\strut #1}}}
\def\PReFlowpair#1#2{%
  \vphantom{\PReFlowbest{0}}%
  \makebox[1.50em][r]{#1}%
  \makebox[0.82em][c]{\textcolor{gray}{$\scriptstyle\to$}}%
  \makebox[1.50em][l]{#2}%
}
\def\PReFlowpairCI#1#2#3#4{%
  \shortstack[c]{%
    \makebox[\PReFlowResultWidth]{%
      \makebox[0.48\PReFlowResultWidth][c]{%
        {\fontsize{3.5}{4.0}\selectfont\textcolor{gray}{[#2]}}}%
      \hfill
      \makebox[0.48\PReFlowResultWidth][c]{%
        {\fontsize{3.5}{4.0}\selectfont\textcolor{gray}{[#4]}}}}%
    \\[-1pt]\PReFlowpair{#1}{#3}}%
}
\def\PReFlowgroup#1{%
  \addlinespace[3pt]
  \multicolumn{12}{@{}l@{}}{%
    {\fontsize{7.1}{8.5}\selectfont\color{gray}\ttfamily\bfseries #1}}\\[-1pt]
}

\fontsize{7.5}{9.0}\selectfont
\setlength{\tabcolsep}{3pt}
\renewcommand{\arraystretch}{1.2}

\newlength{\PReFlowMethodWidth}
\newlength{\PReFlowResultWidth}
\setlength{\PReFlowMethodWidth}{44pt}
\setlength{\PReFlowResultWidth}{%
  \dimexpr(\linewidth-\PReFlowMethodWidth-22\tabcolsep)/11\relax}

\begin{tabular}{%
  @{}>{\raggedright\arraybackslash}m{\PReFlowMethodWidth}%
  *{11}{>{\centering\arraybackslash}m{\PReFlowResultWidth}}@{}}
\toprule

 & \texttt{\textbf{al}} & \texttt{\textbf{ag}} & \texttt{\textbf{hm}} & \texttt{\textbf{hl}}
& \texttt{\textbf{scene}} & \texttt{\textbf{p33}} & \texttt{\textbf{p44}} & \texttt{\textbf{c2}}
& \texttt{\textbf{c3}} & \texttt{\textbf{c4}} & \texttt{\textbf{all}} \\

Method
& {\color{gray}\fontsize{5.5}{6.2}\selectfont 5 tasks}
& {\color{gray}\fontsize{5.5}{6.2}\selectfont 5 tasks}
& {\color{gray}\fontsize{5.5}{6.2}\selectfont 5 tasks}
& {\color{gray}\fontsize{5.5}{6.2}\selectfont 5 tasks}
& {\color{gray}\fontsize{5.5}{6.2}\selectfont 5 tasks}
& {\color{gray}\fontsize{5.5}{6.2}\selectfont 5 tasks}
& {\color{gray}\fontsize{5.5}{6.2}\selectfont 5 tasks}
& {\color{gray}\fontsize{5.5}{6.2}\selectfont 5 tasks}
& {\color{gray}\fontsize{5.5}{6.2}\selectfont 5 tasks}
& {\color{gray}\fontsize{5.5}{6.2}\selectfont 5 tasks}
& {\color{gray}\fontsize{5.5}{6.2}\selectfont 50 tasks} \\
\midrule

\PReFlowgroup{Gaussian}
\texttt{\textbf{ReBRAC}}
 & \PReFlowpairCI{94}{94,95}{\PReFlowbest{98}}{97,98}
 & \PReFlowpairCI{\PReFlowbest{57}}{53,60}{75}{69,79}
 & \PReFlowpairCI{69}{65,74}{58}{52,63}
 & \PReFlowpairCI{17}{15,19}{23}{18,28}
 & \PReFlowpairCI{65}{61,69}{99}{98,99}
 & \PReFlowpairCI{79}{73,84}{\PReFlowbest{100}}{100,100}
 & \PReFlowpairCI{0}{0,0}{0}{0,0}
 & \PReFlowpairCI{9}{8,10}{97}{96,98}
 & \PReFlowpairCI{1}{0,1}{14}{9,20}
 & \PReFlowpairCI{9}{6,11}{\PReFlowbest{80}}{79,80}
 & \PReFlowpairCI{40}{39,41}{64}{63,65} \\

\PReFlowgroup{Guidance}
\texttt{\textbf{QSM}}
 & \PReFlowpairCI{90}{88,92}{94}{91,96}
 & \PReFlowpairCI{24}{19,29}{87}{83,91}
 & \PReFlowpairCI{82}{79,84}{80}{78,81}
 & \PReFlowpairCI{6}{5,7}{2}{1,3}
 & \PReFlowpairCI{78}{77,80}{82}{81,83}
 & \PReFlowpairCI{57}{52,63}{86}{82,90}
 & \PReFlowpairCI{0}{0,0}{98}{94,100}
 & \PReFlowpairCI{33}{31,34}{98}{98,99}
 & \PReFlowpairCI{6}{5,6}{27}{25,28}
 & \PReFlowpairCI{19}{18,19}{67}{62,71}
 & \PReFlowpairCI{39}{38,40}{72}{71,73} \\
\texttt{\textbf{QPILOTS-U}}
 & \PReFlowpairCI{82}{76,88}{96}{94,98}
 & \PReFlowpairCI{1}{0,4}{89}{85,94}
 & \PReFlowpairCI{63}{54,75}{\PReFlowbest{99}}{98,100}
 & \PReFlowpairCI{8}{2,14}{59}{40,72}
 & \PReFlowpairCI{89}{80,95}{98}{96,100}
 & \PReFlowpairCI{\PReFlowbest{100}}{100,100}{\PReFlowbest{100}}{100,100}
 & \PReFlowpairCI{16}{6,25}{\PReFlowbest{100}}{100,100}
 & \PReFlowpairCI{\PReFlowbest{75}}{71,80}{\PReFlowbest{100}}{99,100}
 & \PReFlowpairCI{\PReFlowbest{7}}{4,10}{69}{57,74}
 & \PReFlowpairCI{3}{0,6}{78}{76,80}
 & \PReFlowpairCI{44}{42,47}{89}{87,90} \\
\texttt{\textbf{QPILOTS-M}}
 & \PReFlowpairCI{82}{74,88}{97}{93,98}
 & \PReFlowpairCI{10}{0,14}{90}{85,95}
 & \PReFlowpairCI{71}{64,75}{\PReFlowbest{99}}{98,100}
 & \PReFlowpairCI{12}{6,16}{60}{38,70}
 & \PReFlowpairCI{95}{91,98}{99}{97,100}
 & \PReFlowpairCI{\PReFlowbest{100}}{99,100}{\PReFlowbest{100}}{100,100}
 & \PReFlowpairCI{17}{2,27}{\PReFlowbest{100}}{100,100}
 & \PReFlowpairCI{\PReFlowbest{75}}{70,81}{\PReFlowbest{100}}{99,100}
 & \PReFlowpairCI{\PReFlowbest{7}}{5,9}{72}{67,75}
 & \PReFlowpairCI{10}{3,20}{79}{78,80}
 & \PReFlowpairCI{48}{46,50}{90}{87,91} \\

\PReFlowgroup{Post-processing}
\texttt{\textbf{DSRL}}
 & \PReFlowpairCI{61}{56,66}{90}{88,91}
 & \PReFlowpairCI{3}{1,4}{40}{36,45}
 & \PReFlowpairCI{53}{48,58}{77}{71,82}
 & \PReFlowpairCI{3}{2,5}{29}{22,36}
 & \PReFlowpairCI{\PReFlowbest{99}}{99,100}{\PReFlowbest{100}}{100,100}
 & \PReFlowpairCI{87}{82,92}{87}{78,95}
 & \PReFlowpairCI{0}{0,0}{0}{0,0}
 & \PReFlowpairCI{74}{72,76}{99}{99,99}
 & \PReFlowpairCI{1}{1,2}{0}{0,0}
 & \PReFlowpairCI{2}{2,3}{78}{77,78}
 & \PReFlowpairCI{38}{38,39}{60}{59,61} \\
\texttt{\textbf{FEdit}}
 & \PReFlowpairCI{58}{54,62}{96}{95,96}
 & \PReFlowpairCI{2}{1,3}{86}{83,90}
 & \PReFlowpairCI{22}{20,23}{80}{72,88}
 & \PReFlowpairCI{3}{2,3}{8}{5,11}
 & \PReFlowpairCI{62}{60,65}{99}{98,99}
 & \PReFlowpairCI{99}{98,100}{\PReFlowbest{100}}{100,100}
 & \PReFlowpairCI{34}{30,37}{\PReFlowbest{100}}{100,100}
 & \PReFlowpairCI{40}{37,43}{\PReFlowbest{100}}{100,100}
 & \PReFlowpairCI{2}{2,3}{40}{34,45}
 & \PReFlowpairCI{5}{3,7}{\PReFlowbest{80}}{79,80}
 & \PReFlowpairCI{33}{32,33}{79}{78,80} \\
\texttt{\textbf{IFQL}}
 & \PReFlowpairCI{36}{32,39}{76}{71,81}
 & \PReFlowpairCI{1}{0,2}{17}{16,18}
 & \PReFlowpairCI{86}{85,87}{51}{47,56}
 & \PReFlowpairCI{24}{21,27}{7}{5,9}
 & \PReFlowpairCI{84}{83,85}{96}{96,97}
 & \PReFlowpairCI{\PReFlowbest{100}}{100,100}{\PReFlowbest{100}}{99,100}
 & \PReFlowpairCI{0}{0,0}{0}{0,0}
 & \PReFlowpairCI{11}{10,12}{79}{79,80}
 & \PReFlowpairCI{0}{0,0}{0}{0,0}
 & \PReFlowpairCI{2}{1,3}{61}{58,63}
 & \PReFlowpairCI{34}{34,35}{49}{48,50} \\

\PReFlowgroup{Intermediate-value learning}
\texttt{\textbf{Q-Flow}}
 & \PReFlowpairCI{\PReFlowbest{95}}{94,96}{\PReFlowbest{98}}{97,99}
 & \PReFlowpairCI{36}{32,39}{54}{51,58}
 & \PReFlowpairCI{87}{86,89}{96}{95,96}
 & \PReFlowpairCI{8}{7,10}{10}{9,12}
 & \PReFlowpairCI{97}{96,98}{99}{99,100}
 & \PReFlowpairCI{\PReFlowbest{100}}{100,100}{\PReFlowbest{100}}{100,100}
 & \PReFlowpairCI{0}{0,1}{0}{0,1}
 & \PReFlowpairCI{37}{34,39}{99}{99,100}
 & \PReFlowpairCI{3}{2,4}{25}{17,34}
 & \PReFlowpairCI{15}{11,19}{\PReFlowbest{80}}{80,80}
 & \PReFlowpairCI{48}{47,48}{66}{65,67} \\
\texttt{\textbf{RQL}}
 & \PReFlowpairCI{82}{80,84}{90}{88,91}
 & \PReFlowpairCI{41}{36,45}{58}{54,63}
 & \PReFlowpairCI{\PReFlowbest{96}}{92,99}{\PReFlowbest{99}}{98,99}
 & \PReFlowpairCI{\PReFlowbest{40}}{38,43}{51}{48,53}
 & \PReFlowpairCI{87}{85,89}{98}{97,98}
 & \PReFlowpairCI{\PReFlowbest{100}}{100,100}{\PReFlowbest{100}}{100,100}
 & \PReFlowpairCI{32}{21,43}{\PReFlowbest{100}}{100,100}
 & \PReFlowpairCI{18}{17,20}{91}{89,93}
 & \PReFlowpairCI{3}{2,4}{38}{33,43}
 & \PReFlowpairCI{\PReFlowbest{49}}{44,53}{\PReFlowbest{80}}{79,80}
 & \PReFlowpairCI{\PReFlowbest{55}}{53,56}{80}{80,81} \\

\PReFlowgroup{Adjoint matching}
\texttt{\textbf{QAM}}
 & \PReFlowpairCI{81}{78,84}{\PReFlowbest{98}}{97,98}
 & \PReFlowpairCI{18}{14,22}{53}{52,54}
 & \PReFlowpairCI{67}{64,69}{83}{78,88}
 & \PReFlowpairCI{11}{9,14}{21}{17,24}
 & \PReFlowpairCI{97}{96,98}{\PReFlowbest{100}}{100,100}
 & \PReFlowpairCI{\PReFlowbest{100}}{99,100}{\PReFlowbest{100}}{100,100}
 & \PReFlowpairCI{0}{0,0}{3}{0,8}
 & \PReFlowpairCI{64}{62,66}{\PReFlowbest{100}}{100,100}
 & \PReFlowpairCI{3}{3,4}{70}{64,76}
 & \PReFlowpairCI{3}{2,4}{71}{67,74}
 & \PReFlowpairCI{44}{44,45}{70}{69,71} \\
\texttt{\textbf{QAM-F}}
 & \PReFlowpairCI{83}{81,84}{93}{92,94}
 & \PReFlowpairCI{12}{8,16}{72}{70,74}
 & \PReFlowpairCI{65}{62,68}{85}{80,90}
 & \PReFlowpairCI{12}{9,14}{20}{16,23}
 & \PReFlowpairCI{95}{92,97}{\PReFlowbest{100}}{99,100}
 & \PReFlowpairCI{99}{99,100}{\PReFlowbest{100}}{100,100}
 & \PReFlowpairCI{6}{5,8}{\PReFlowbest{100}}{100,100}
 & \PReFlowpairCI{65}{63,67}{\PReFlowbest{100}}{100,100}
 & \PReFlowpairCI{3}{2,3}{68}{64,73}
 & \PReFlowpairCI{14}{11,17}{79}{79,80}
 & \PReFlowpairCI{45}{45,46}{82}{81,82} \\
\texttt{\textbf{QAM-E}}
 & \PReFlowpairCI{83}{80,86}{96}{95,97}
 & \PReFlowpairCI{1}{0,3}{\PReFlowbest{95}}{91,97}
 & \PReFlowpairCI{59}{54,63}{83}{75,90}
 & \PReFlowpairCI{2}{1,3}{16}{11,22}
 & \PReFlowpairCI{97}{96,98}{\PReFlowbest{100}}{100,100}
 & \PReFlowpairCI{\PReFlowbest{100}}{100,100}{\PReFlowbest{100}}{100,100}
 & \PReFlowpairCI{39}{35,43}{99}{97,100}
 & \PReFlowpairCI{65}{63,68}{\PReFlowbest{100}}{100,100}
 & \PReFlowpairCI{5}{4,6}{\PReFlowbest{79}}{79,79}
 & \PReFlowpairCI{6}{4,9}{79}{79,80}
 & \PReFlowpairCI{46}{45,47}{85}{84,86} \\
\texttt{\textbf{TRQAM}}
 & \PReFlowpairCI{82}{80,85}{97}{96,97}
 & \PReFlowpairCI{14}{12,15}{92}{90,94}
 & \PReFlowpairCI{28}{26,30}{76}{73,79}
 & \PReFlowpairCI{2}{1,3}{7}{5,9}
 & \PReFlowpairCI{72}{66,80}{99}{99,100}
 & \PReFlowpairCI{\PReFlowbest{100}}{100,100}{\PReFlowbest{100}}{100,100}
 & \PReFlowpairCI{0}{0,0}{0}{0,0}
 & \PReFlowpairCI{\PReFlowbest{75}}{73,76}{\PReFlowbest{100}}{100,100}
 & \PReFlowpairCI{5}{4,6}{76}{73,78}
 & \PReFlowpairCI{3}{2,5}{76}{76,77}
 & \PReFlowpairCI{38}{37,39}{72}{72,73} \\

\PReFlowgroup{Ours}
\texttt{\textbf{PReFlow}}
 & \PReFlowpairCI{82}{79,85}{\PReFlowbest{98}}{98,99}
 & \PReFlowpairCI{12}{8,15}{94}{91,97}
 & \PReFlowpairCI{65}{60,69}{95}{91,99}
 & \PReFlowpairCI{24}{20,28}{\PReFlowbest{65}}{58,72}
 & \PReFlowpairCI{97}{96,97}{99}{98,99}
 & \PReFlowpairCI{\PReFlowbest{100}}{100,100}{\PReFlowbest{100}}{100,100}
 & \PReFlowpairCI{\PReFlowbest{66}}{58,74}{\PReFlowbest{100}}{100,100}
 & \PReFlowpairCI{59}{56,62}{\PReFlowbest{100}}{100,100}
 & \PReFlowpairCI{\PReFlowbest{7}}{5,8}{75}{74,77}
 & \PReFlowpairCI{15}{12,17}{\PReFlowbest{80}}{80,80}
 & \PReFlowpairCI{53}{51,54}{\PReFlowbest{91}}{90,92} \\
\bottomrule
\end{tabular}
\endgroup
\caption{
\textbf{OGBench: offline $\rightarrow$ online success rates (\%).}
Each cell shows end-of-offline (left) and end-of-online (right)
success rates with 95\% CIs over seeds in small gray brackets. \best{Best results per domain are highlighted}. Per-task results are in Appendix~\ref{app:full-results}.
}
\label{tab:main_combined}

\end{table}

\begin{figure}[!t]
    \centering
    \includegraphics[width=\linewidth]{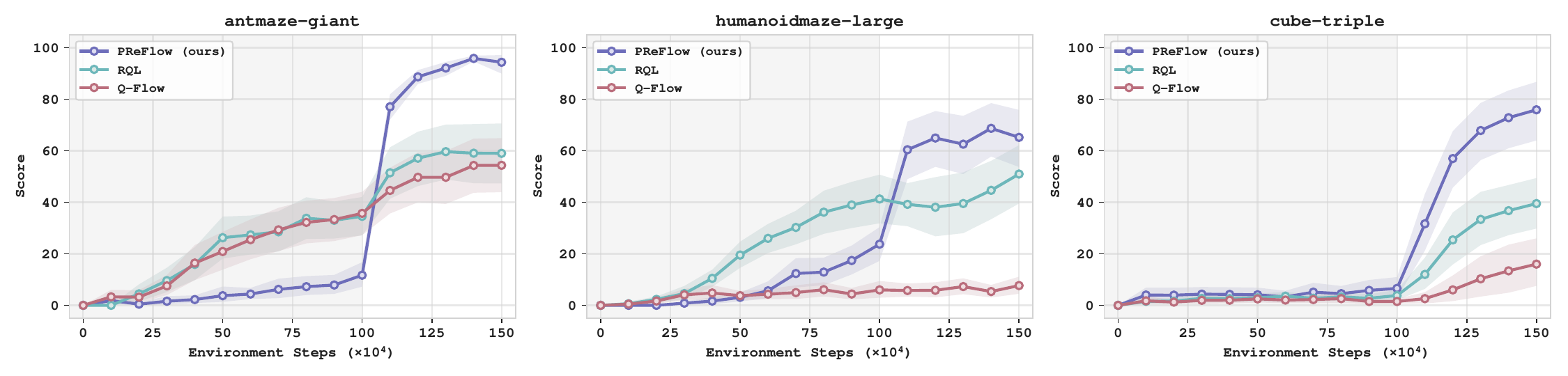}
    \captionsetup{skip=4pt}
    \caption{
\textbf{Learning curves during online fine-tuning.}
PReFlow, RQL, and Q-Flow are fine-tuned for 0.5M online environment
steps after 1M offline gradient steps.
Curves show means over three seeds; shaded regions show pointwise
95\% confidence intervals.
}
    \label{fig:online_adaption}

\end{figure}

Table~\ref{tab:main_combined} reports offline and offline-to-online
performance.
PReFlow achieves an aggregate offline score of 53, competitive
with RQL and above all evaluated adjoint matching and
test-time steering baselines.
On \texttt{p44}, it exceeds the strongest baseline by about
27 percentage points.
PReFlow also outperforms FEdit and IFQL, which use Gaussian
editing and selection among behavior samples.

After 500K online environment steps, PReFlow reaches an aggregate
score of 91, the highest mean among the compared methods.
Figure~\ref{fig:online_adaption} shows that PReFlow outperforms
RQL and Q-Flow after fine-tuning on \texttt{ag},\texttt{hl}, and \texttt{c3} despite relatively low offline scores.
Unlike intermediate value learning methods like RQL and Q-Flow, PReFlow trains its refinement flow directly
from critic gradients at the final action. This may help the refiner respond more directly to critic updates
from online data, while critic-based proposal selection allows PReFlow to further benefit from these updates.

\subsection{Expressiveness and Efficiency of Conditional Refinement Flow}
\label{sec:exp-refinement}

\par
\begingroup
\setlength{\columnsep}{14pt}
\setlength{\intextsep}{-0.45pt}

\begin{wraptable}{r}{0.3\textwidth}
    \centering
    \scriptsize
    \resizebox{\linewidth}{!}{%
    \renewcommand{\arraystretch}{1.1}
    \setlength{\tabcolsep}{1.5pt}
    \begin{tabular}{@{}lccc@{}}
    \toprule
    &
    & \multicolumn{2}{c}{\color{gray} \texttt{\textbf{Refinement Type}}} \\
    \cmidrule(lr){3-4}
    Domain
    & \texttt{\textbf{Base}}
    & \texttt{\textbf{Gaussian}}
    & \texttt{\textbf{Flow}} \\
    \midrule
    \textbf{\texttt{al}}
    & 2.3{\color{gray}\tiny$\pm$0.8}
    & 76.3{\color{gray}\tiny$\pm$7.2}
    & \cellbg 79.0{\color{gray}\tiny$\pm$7.5} \\
    \textbf{\texttt{hm}}
    & 2.7{\color{gray}\tiny$\pm$1.0}
    & 57.1{\color{gray}\tiny$\pm$8.5}
    & \cellbg 69.2{\color{gray}\tiny$\pm$1.2} \\
    \textbf{\texttt{scene}}
    & 3.9{\color{gray}\tiny$\pm$1.5}
    & 68.1{\color{gray}\tiny$\pm$3.2}
    & \cellbg 73.1{\color{gray}\tiny$\pm$0.8} \\
    \textbf{\texttt{c2}}
    & 2.0{\color{gray}\tiny$\pm$0.0}
    & 40.8{\color{gray}\tiny$\pm$5.2}
    & \cellbg 53.9{\color{gray}\tiny$\pm$3.3} \\
    \bottomrule
\end{tabular}
}
    \captionsetup{font=footnotesize,skip=4pt,singlelinecheck=false}
    \caption{
    Offline success rates (\%) after 1M steps with $K=1$.
    }
    \label{tab:abl-edit-family}
\end{wraptable}

\paragraph{Refinement policy class on OGBench.}
We compare no refinement, a learned Gaussian, and PReFlow on four OGBench domains:
\texttt{al}, \texttt{hm},
\texttt{scene}, and \texttt{c2}.
We use a single proposal, and set
$\delta=0.1$ for refinement scale.
Both Gaussian and flow refinement policies are trained under
the conditional objective in~\eqref{eq:conditional-objective}.
Table~\ref{tab:abl-edit-family} shows that both refiners
substantially improve over the base policy.
The flow achieves higher mean scores than the Gaussian across
all four domains, supporting the benefits of expressive
conditional refinement.
\par
\endgroup

\par
\begingroup
\setlength{\columnsep}{14pt}
\setlength{\intextsep}{0pt}

\begin{wrapfigure}{r}{0.3\textwidth}

    \centering
    \includegraphics[width=\linewidth]{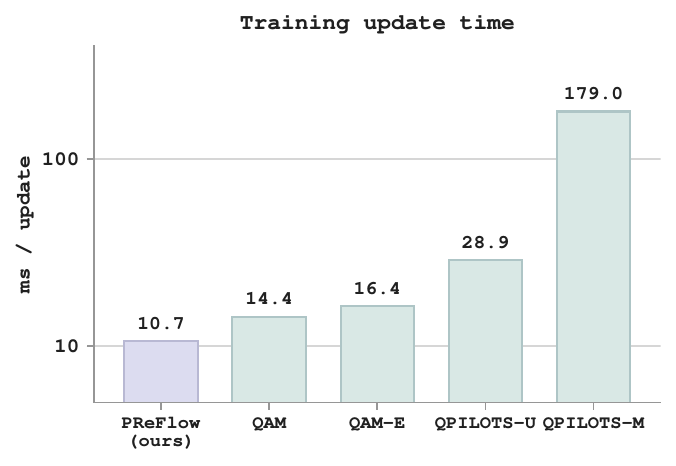}
    \captionsetup{skip=4pt,singlelinecheck=false}
    \caption{
    Mean wall-clock time per full training update.
    }
    \label{fig:flowstep-cost}
\end{wrapfigure}

\paragraph{Computational cost.}
\label{sec:exp-cost}
We next examine the computational cost of learning the flow refinement policy. Figure~\ref{fig:flowstep-cost} compares the
wall-clock time per full training update with QAM, QAM-E,
QPILOTS-U, and QPILOTS-M on GPU,
averaged over five flow-step counts: 5, 10, 20, 50, and 100.
PReFlow has the lowest mean update time among the compared methods. 
Its regression target is constructed in closed form without any backward simulation, whereas QAM requires sequential backward adjoint
propagation. QPILOTS requires per-step steering with additional estimators to sample actions.

\par
\endgroup

\subsection{Effect of Proposal Selection}
\label{sec:exp-selection}

The policy-class comparison above fixes $K=1$.
We now vary $K\in\{1,2,4,8\}$ while retaining learned refinement
in every setting. For $K>1$, PReFlow selects the behavior proposal
with the highest critic value before refinement, while $K=1$ directly
refines a single sampled proposal.

Figure~\ref{fig:selection_ablation} presents results on
\texttt{c2}, \texttt{scene}, and \texttt{hm}.
Overall, proposal selection complements learned refinement,
with gains in offline performance and final online success rates
already evident with $K=2$.
Full results for all ten OGBench domains are provided in
Appendix~\ref{app:abl-selection}. We further compare proposal selection with and without refinement
in Appendix~\ref{app:abl-norefine}, showing selection alone improves over the behavior policy but remains far below refinement even with a single proposal.

\subsection{Sensitivity to Refinement Scale and Inverse Temperature}
\label{sec:exp-sensitivity}

We study how the refinement scale $\delta$ and
inverse temperature $\tau$ affect PReFlow's performance.
As shown in Figure~\ref{fig:sensitivity}, we find that performance
varies with both $\delta$ and $\tau$.
Under a local linear approximation, the conditional
target induces an action refinement with mean
$\tau\delta^2\nabla_a Q(s,a)$ and covariance $\delta^2 I$.
Both $\delta$ and $\tau$ therefore control the effective refinement strength.

\begin{figure}[t]
    \centering
    \includegraphics[width=\linewidth]{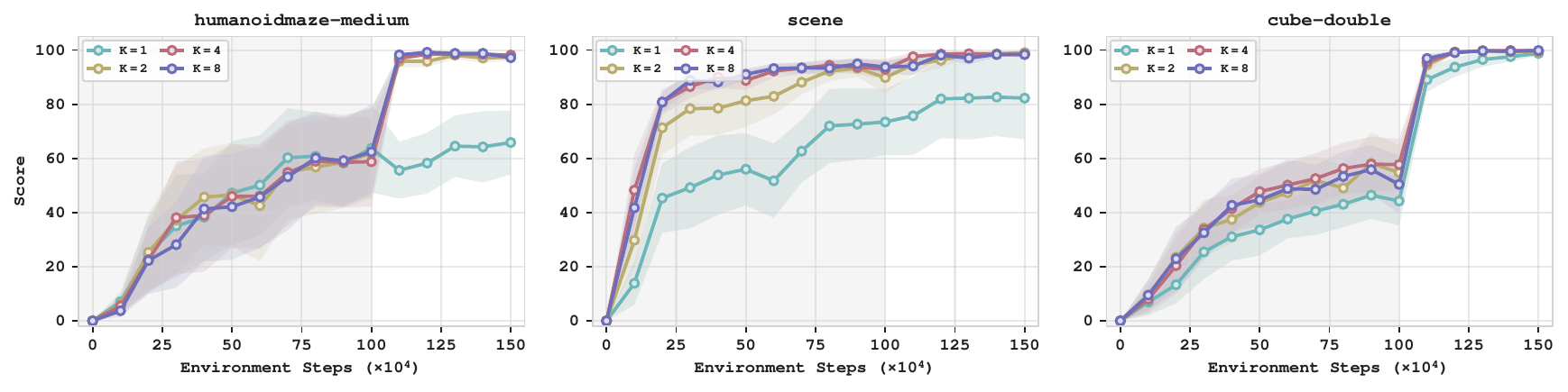}
    \captionsetup{skip=4pt}
    \caption{
    \textbf{Effect of proposal selection.}
    PReFlow with $K\in\{1,2,4,8\}$ behavior proposals on \texttt{hm}, 
    \texttt{scene}, and \texttt{c2}.
    All settings use the learned flow to refine the proposal with the highest critic value. Curves show means over three seeds; shaded regions show pointwise
95\% confidence intervals.
    }
    \label{fig:selection_ablation}

\end{figure}

\begin{figure}
    \centering
    \includegraphics[width=0.9\linewidth]{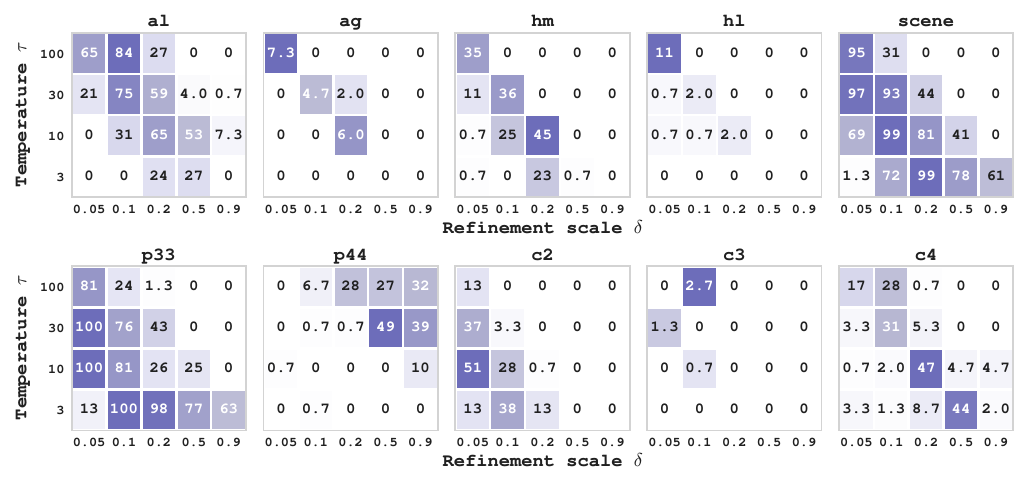}
    \captionsetup{skip=4pt}
    \caption{
\textbf{Sensitivity to the refinement scale and inverse temperature.}
Offline performance on OGBench domains across different combinations of
the refinement scale $\delta$ and inverse temperature $\tau$.
Each cell reports the mean performance over three seeds.
}
\label{fig:sensitivity}

\end{figure}

%% file: sections/7_conclusion.tex
\section{Conclusion}
\label{sec:conclusion} 

We present PReFlow, a policy extraction method that selects among behavior proposals with the critic and refines the selected proposal with a conditional flow.
A proposal-centered Gaussian reference encourages small action changes, while the flow can represent several high-value refinements of the same proposal. The same reference provides closed-form adjoint matching targets
from sampled refinement endpoints and critic gradients, allowing simulation-free training of the flow without a backward adjoint solve. On 50 OGBench tasks, PReFlow is competitive in offline RL and attains the highest aggregate success rate after online fine-tuning among the baselines.
Our results suggest that expressive policies help not only in modeling multimodal behavior but also in refining each proposal.

\paragraph{Limitation.} Proposal selection uses $Q_\psi$ as an approximation to the refinement value $V_\delta$, but proposals with similar critic values can differ in $V_\delta$ through the local gradient and curvature of the critic.
Efficient estimation of $V_\delta$ that captures this dependence is a natural direction for future work.

%% file: sections/8_statements.tex





%% file: sections/9_appendix.tex
\section{Proofs and Technical Details}
\label{app:proofs}

We derive the optimal refinement and selection policies, the adjoint
target, and the proposal-selection approximation. We use the notation
of the main text, with $\tau=\tau(s)>0$ and $\delta>0$.

\subsection{Optimal proposal-conditioned policy improvement}
\label{app:lift-proof}

Suppose $Q_\psi(s,\cdot)$ is bounded. For the proofs, define
\begin{equation}
Z(s,a)=\E_{\Delta a\sim\rho_\delta}
\left[e^{\tau Q_\psi(s,a+\Delta a)}\right],
\quad Z_J(s)=\E_{a\sim\pi_\beta}[Z(s,a)].
\label{eq:app-partition-functions}
\end{equation}
Both constants are finite and strictly positive. Since the critic is considered bounded, it is sufficient to consider policies with finite KL penalties.
Substituting the normalized form of \eqref{eq:refinement-target} into
\eqref{eq:conditional-objective} gives
\begin{equation}
\mathcal J_{s,a}(\pi^{\mathrm{ref}})
=\frac1\tau\log Z(s,a)-\frac1\tau
\operatorname{KL}\!\left(\pi^{\mathrm{ref}}\,\middle\|\,
\pi^{\mathrm{ref},\star}\right).
\label{eq:app-cond-gibbs}
\end{equation}
Thus $\pi^{\mathrm{ref},\star}$ is optimal, with
$V_\delta(s,a)=\tau^{-1}\log Z(s,a)$.

\begin{proof}[Proof of Proposition~\ref{prop:lift}]
After optimizing conditional refinement, \eqref{eq:joint-objective}
reduces to
\begin{equation}
    \E_{a\sim\pi^{\mathrm{base}}}[V_\delta(s,a)]
-\tau^{-1}\operatorname{KL}(\pi^{\mathrm{base}}\|\pi_\beta).
\end{equation}
Applying the same Gibbs variational identity gives the optimal
selection policy
\begin{equation}
\pi^{\mathrm{base},\star}(a\mid s)
=\frac{\pi_\beta(a\mid s)e^{\tau V_\delta(s,a)}}{Z_J(s)}
=\frac{\pi_\beta(a\mid s)Z(s,a)}{Z_J(s)},
\label{eq:app-proposal-marginal}
\end{equation}
Substituting this policy and the
normalized refinement target into the final-action mixture yields
\begin{equation}
\pi_\delta^\star(\tilde a\mid s)
=\int \pi^{\mathrm{base},\star}(a\mid s)
\pi^{\mathrm{ref},\star}(\tilde a-a\mid s,a)\,\dd a
\quad
=\frac{e^{\tau Q_\psi(s,\tilde a)}}{Z_J(s)}
\E_{a\sim\pi_\beta}\left[\rho_\delta(\tilde a-a)\right].
\label{eq:app-final-action-marginal}
\end{equation}
Expanding the Gaussian density $\rho_\delta$ gives
\eqref{eq:final-policy}, which characterizes the unclipped action
distribution.
\end{proof}

\subsection{Gaussian reference and adjoint target}
\label{app:control-details}

As in Section~\ref{sec:am}, write $\Delta a=\delta y$,  $\rho=\N(0,I)$.
We use the $C=1$ member of the linear reference family in
\citet[Proposition~3.1]{eam}. With $d(t)=t^2+(1-t)^2$, its drift and
diffusion coefficient are
\begin{equation}
b(y,t)=D(t)y,
\qquad D(t)=\frac{2t^2-1}{t\,d(t)},
\qquad \sigma(t)^2=\frac{2(1-t)}t,
\quad 0<t<1.
\label{eq:app-gaussian-reference}
\end{equation}
Its marginals are $\N(0,d(t)I)$, its terminal law is $\rho$, and its
endpoint-conditioned marginal is the bridge in
\eqref{eq:reference-bridge}. The singular coefficients at $t=0$ are
interpreted through the Gaussian entrance law, see
\citet[Appendix~B.3]{eam} for the reference construction.

Since the terminal law matches the Gaussian prior in
\eqref{eq:conditional-target}, the density correction in
\citet[Proposition~3.2]{eam} vanishes. The terminal cost is therefore
$g_{s,a}(y)=-\tau Q_\psi(s,a+\delta y)$.
For the controlled SDE
$\dd Y_t=(b(Y_t,t)+\sigma(t)u_t)\dd t+\sigma(t)\dd W_t$,
the memoryless SOC formulation yields the optimal endpoint law
$q^\star$ in \eqref{eq:conditional-target}, under the usual
well-posedness and change-of-measure conditions and a realizable
terminal tilt~\citep[Section~4 and Appendices~C.4 and~D]{am}.

The reference velocity is $v_0(y,t)=(2t-1)y/d(t)$, and we parameterize
the control as
\begin{equation}
u_\phi(y,t\mid s,a)
=\frac2{\sigma(t)}\bigl(v_\phi(y,t\mid s,a)-v_0(y,t)\bigr).
\label{eq:residual-control-velocity}
\end{equation}
This gives controlled drift $2v_\phi-y/t$. Following Adjoint
Sampling~\citep{adjointsampling}, we hold the sampled endpoint fixed
during the update and draw intermediate states from
\eqref{eq:reference-bridge}.

\begin{proof}[Proof of Proposition~\ref{prop:am}]
The reference Jacobian is $\nabla_y b(y,t)=D(t)I$.
Thus \eqref{eq:lean-adjoint} gives
$\dot\lambda(t)=-D(t)\lambda(t)$, with
$\lambda(1)=-\tau\nabla_yQ_\psi(s,a+\delta y)|_{y=y_1}$.
Using $D(t)=\frac{\dd}{\dd t}\log(d(t)/t)$ and $d(1)=1$, we obtain
\begin{equation}
\lambda(t)
=\exp\!\left(\int_t^1D(r)\,\dd r\right)\lambda(1)
=\frac{t}{d(t)}\lambda(1).
\label{eq:app-adjoint}
\end{equation}
The AM control target is $-\sigma(t)\lambda(t)$. Converting it to a
velocity through \eqref{eq:residual-control-velocity} yields
\begin{align}
v^{\mathrm{tgt}}(y_t,y_1,t\mid s,a)
&=v_0(y_t,t)-\frac{\sigma(t)^2}{2}\lambda(t)\nonumber\\
&=\frac{2t-1}{d(t)}y_t
+\frac{1-t}{d(t)}\tau
\left.\nabla_y Q_\psi(s,a+\delta y)\right|_{y=y_1}.
\label{eq:app-velocity-target}
\end{align}
Finally, the chain rule gives
$\nabla_yQ_\psi(s,a+\delta y)|_{y=y_1}
=\delta\nabla_{\tilde a}Q_\psi(s,\tilde a_1)$,
which establishes \eqref{eq:vtgt}.
\end{proof}

The control-to-velocity conversion also determines the AM loss weight:
\begin{equation}
\left\|u_\phi+\sigma(t)\lambda(t)\right\|^2
=\frac{4}{\sigma(t)^2}
\left\|v_\phi-v^{\mathrm{tgt}}\right\|^2.
\label{eq:app-velocity-weight}
\end{equation}
The weight is $4/\sigma(t)^2=2t/(1-t)$. Following \citet[Appendix~B.2]{ram}, we omit the weight and sample time uniformly in \eqref{eq:actor-loss}. 
For a frozen endpoint distribution and finite losses, positive time
weights preserve the pointwise conditional-mean minimizer in an
unrestricted velocity class, but can change the fit of a finite network.

The idealized Adjoint Sampling characterization requires both
population regression stationarity and consistency of the learned
dynamics with the reference bridges~\citep[Theorem~C.3]{adjointsampling}.
Here, consistency requires the ODE velocity to equal the flow-matching
velocity associated with its own endpoint distribution and
\eqref{eq:reference-bridge}. The target derivation alone establishes
neither this consistency nor convergence of finite-network training
or finite-step sampling to $q^\star$.

\subsection{Proposal-selection approximation}
\label{app:proposal-selection}

We justify replacing $V_\delta(s,a)$ by the critic value in selection
weights. Keep $\tau$ fixed and suppose $Q_\psi(s,\cdot)$ is bounded and
twice continuously differentiable near $a$.
For $F(x)=e^{\tau Q_\psi(s,x)}$ and $Y\sim\rho$, a second-order Taylor
expansion gives
\begin{equation}
\E[F(a+\delta Y)]
=F(a)+\frac{\delta^2}{2}\tr(\nabla_a^2F(a))+o(\delta^2).
\label{eq:app-partition-expansion}
\end{equation}
The linear term vanishes because $\E[Y]=0$, and the quadratic term
uses $\E[YY^\top]=I$. Continuity of the Hessian controls the local
remainder, and boundedness of $F$ and Gaussian tails make the contribution
outside that neighborhood $o(\delta^2)$.
Taking the logarithm and dividing by $\tau$ yields
\begin{equation}
V_\delta(s,a)
=Q_\psi(s,a)+\frac{\delta^2}{2}
\left[\tr(\nabla_a^2Q_\psi(s,a))
+\tau\|\nabla_aQ_\psi(s,a)\|^2\right]
+o(\delta^2).
\label{eq:app-local-value-expansion}
\end{equation}
Thus $V_\delta(s,a)=Q_\psi(s,a)+O(\delta^2)$, motivating the
leading-order weights $e^{\tau Q_\psi(s,a)}$ for small refinement scales.
For i.i.d.\ behavior proposals, weights
$Z(s,a_i)=e^{\tau V_\delta(s,a_i)}$ give a self-normalized
importance-sampling approximation to $\pi^{\mathrm{base},\star}$.
Since the weights are bounded and positive, the strong law of large
numbers gives almost-sure convergence of weighted expectations of
bounded test functions to their expectations under
$\pi^{\mathrm{base},\star}$ as the number of candidates increases.
Replacing $V_\delta$ by $Q_\psi$ gives the finite-temperature weights
$w_i\propto e^{\tau Q_\psi(s,a_i)}$ used in Section~\ref{sec:lift}.

Our implementation selects $I=\arg\max_i Q_\psi(s,a_i)$ before
refinement. Table~\ref{tab:abl-selection} compares greedy and
finite-temperature selection empirically.

\section{Extended Related Work}
\label{app:related}

\subsection{Policy extraction with diffusion and flow policies}

Q-Flow and RQL differ in how they learn values over intermediate
generative states. Q-Flow~\citep{qflow} fits an auxiliary value
network to critic evaluations at endpoints obtained by integrating
the current flow from intermediate states. Its policy update then
uses gradients of this intermediate value network.
RQL~\citep{rql} reconstructs virtual trajectories by running the flow
backward from dataset actions and uses multi-step returns to learn
values in the expanded MDP. PReFlow evaluates $Q_\psi$ at the final
refined action and uses its gradient to construct a velocity target
through the analytic reference adjoint. Its actor update requires
neither an intermediate value estimator nor value backups across
refinement steps.

QPILOTS~\citep{qpilots} estimates guidance from clean actions
conditioned on an intermediate flow state. QPILOTS-U uses a denoised
point estimate, whereas QPILOTS-M trains a meta flow map to produce
approximate posterior samples. The auxiliary model in QPILOTS-M
therefore supports critic-gradient estimation during sampling.
PReFlow's additional flow models refinements conditioned on a
behavior proposal. The critic gradient supervises this flow during
training, so the refinement rollout requires no critic gradients
at inference time.

\subsection{Behavior proposals and conditional action refinement}

SPAR~\citep{spar} evaluates flow-based residual policies in addition
to its conditional VAE formulation. Its latent self-imitation
procedure samples residual candidates from a target policy, weights
them using conservative advantage estimates, and fits the residual
policy to these weighted samples. Both the weights and sampled
residuals are detached during this update.
PReFlow specifies a conditional Gibbs target through its
proposal-centered KL objective. The critic gradient at a sampled
refined action determines a detached velocity target through
the analytic adjoint.

FLAG~\citep{flag} uses a local Gaussian policy around each
flow-generated anchor. With fixed covariance, the EM update moves
each anchor toward the mean of a reweighted local target.
Its expressive action marginal arises from mixing these local
policies over flow latents.
PReFlow's conditional flow can represent multiple refinement modes
for the same proposal; its Gaussian reference defines the KL
penalty rather than the policy family.

The regularization of refinements also differs across methods.
DeFlow constrains the displacement of a deterministic
refiner~\citep{deflow}, while Fisher Decorator penalizes a local
transport map using the behavior score~\citep{fidec}.
PReFlow's Gaussian KL regularizer balances expected squared
displacement against conditional entropy. It discourages large
edits without imposing a hard radius on individual refinements or
requiring an estimate of the behavior score.

EMaQ and decoupled policy extraction use critic values to select
among behavior proposals~\citep{emaq,dpe}.
PReFlow's joint objective weights proposals by their optimal
regularized refinement value, accounting for the value and KL cost
of subsequent refinement. At the joint optimum, the unclipped
action marginal is a Gibbs policy under the Gaussian-smoothed
behavior prior. The implementation uses critic-based argmax selection;
this rule does not in general sample the optimal proposal marginal
(Appendix~\ref{app:proposal-selection}).

\subsection{Adjoint matching and tractable reference dynamics}

QAM applies Adjoint Matching~\citep{am,qam} with a learned behavior
flow as the reference. Although its actor update avoids
backpropagation through the optimized policy's sampler, target
construction requires a backward lean-adjoint solve through the
behavior reference. TRQAM~\citep{dong2026trustregionqadjoint} adapts
the constraint through a path-space trust region, while
ME-AM~\citep{meam} modifies entropy regularization and introduces
a mixture behavior prior.

EAM~\citep{eam} constructs linear reference dynamics with
closed form lean adjoints. To preserve the original reward-tilted
pretrained distribution, it corrects the terminal cost for the
change of reference; the gradient of this correction depends on
the pretrained model's score.
PReFlow uses the $C=1$ member of EAM's reference family.
Its terminal law $\rho=\mathcal N(0,I)$ already matches the prior
in our conditional refinement objective, so the density correction
vanishes. The terminal cost is therefore
$-\tau Q_\psi(s,a+\delta y)$, while behavior information enters through
the proposals.

Adjoint Sampling~\citep{adjointsampling} uses reference bridges to
separate endpoint generation from the construction of intermediate
training samples. AMDP~\citep{maxentam} applies reciprocal adjoint
matching to maximum-entropy RL.
PReFlow uses this training construction for its conditional
refinement objective. Proposals and refinement endpoints are
generated by detached forward rollouts. Given an endpoint, the
analytic reference bridge and lean adjoint provide intermediate
samples and velocity targets without further trajectory simulation
or a backward solve. The simulation-free property applies to these
operations; generating the endpoints still requires forward rollouts.

Appendix~\ref{app:control-details} derives our velocity target and
time weighting, and states the consistency conditions underlying the
idealized Adjoint Sampling characterization~\citep[Theorem~C.3]{adjointsampling}.

\section{Additional Method Details}
\label{app:additional}

\subsection{Temperature Adaptation}
\label{app:autotau}

A fixed $\tau$ can lead to different deviations from the reference
across domains because the critic gradients have different scales.
We use a cheap local approximation of the conditional KL to guide
temperature adaptation.

\paragraph{Local KL approximation.}
For fixed $s$ and $a$, consider the linear approximation
$
Q_\psi(s,a+\delta y)\approx Q_\psi(s,a)+\delta g_a^\top y,
\;
g_a=\nabla_a Q_\psi(s,a).
$
With the Gaussian reference $\rho=\N(0,I)$, the corresponding
conditional target is $q^\star_{\mathrm{lin}}=\N(\tau\delta g_a,I)$.
Writing $d_{\mathrm{act}}=\dim(\cA)$, its KL per action dimension is
$
\frac{1}{d_{\mathrm{act}}}\operatorname{KL}(q^\star_{\mathrm{lin}}\|\rho)
=\frac{\tau^2\delta^2\|g_a\|^2}{2d_{\mathrm{act}}}
=\frac{1}{2}\left(
\frac{\E_{y\sim q^\star_{\mathrm{lin}}}\|y\|^2}{d_{\mathrm{act}}}-1
\right).
$
Under this approximation, estimating the KL requires only squared
endpoint norms. We use this expression as a cheap KL proxy for the learned refinement policy.

\paragraph{Temperature update.}
For endpoints sampled from the current
refinement policy, define the normalized second moment, or dispersion,
as
$
\mathrm{disp}=\frac{1}{d_{\mathrm{act}}}\E\|y_1\|^2.
$
The reference $\rho$ has dispersion one, regardless of the action
dimension or chunk length. For the linearized target, a dispersion
target $c\ge1$ therefore corresponds to a conditional KL of
$(c-1)/2$ per action dimension. We adjust $\tau$ to keep the measured dispersion close to $c$ using
$
\mathcal L_\tau
=\log\tau\left(\sg[\mathrm{disp}]-c\right).
$
We initialize $\tau$ to $10$ on \texttt{al}, \texttt{ag}, and
\texttt{p44}, and to $30$ on \texttt{c4}. The temperature is updated at every gradient step.

\subsection{Implementation details}
\label{app:impl}

\paragraph{Critic and behavior-flow updates.}
We follow QAM~\citep{qam} for the critic architecture, ensemble aggregation,
Bellman backup construction, and target-network updates. The behavior-flow
initialization and training schedule during offline and online learning
also follow QAM. Critic values and action gradients use the same ensemble
aggregation rules as QAM.
We train the critic with the same expectile regression objective as
RQL~\citep{rql}, using the domain-specific $\kappa$ values in
Appendix~\ref{app:hparams}.
Training proposals for the refinement-flow update are drawn directly
from $\pi_\beta$, without proposal selection, as in
\eqref{eq:actor-loss}. For each TD target, we also sample one proposal
from $\pi_\beta$ and refine it with $q_\phi$.
The $K$-candidate selection rule in
Algorithm~\ref{algo:preflow-action} is used for action generation.

\paragraph{Action selection and clipping.}
A single batched behavior-flow rollout generates the $K$ candidate
proposals.
Our benchmark policy chooses the candidate with the highest critic
value and samples a refinement conditioned on that proposal.
When $K=1$, this reduces to refining the only available proposal.
The resulting action $\tilde a=a+\delta y$ is clipped to the
environment's action bounds. 

\subsection{Training procedure}
\label{app:training-procedure}

Algorithm~\ref{algo:preflow-full} summarizes one joint update of the
critic, behavior flow, and refinement flow.
Following QAM~\citep{qam}, we use ten critics
$\{Q_{\psi^j}\}_{j=1}^{10}$ and their target networks
$\{Q_{\bar\psi^j}\}_{j=1}^{10}$.
Let $\bar Q_{\mathrm{mean}}$ and $\bar Q_{\mathrm{std}}$ denote
the mean and standard deviation of the target ensemble.
The pessimistic target value is
\[
    \bar Q(s,a)
    =
    \bar Q_{\mathrm{mean}}(s,a)
    -\rho_Q\bar Q_{\mathrm{std}}(s,a),
\]
where $\rho_Q$ corresponds to QAM's pessimism coefficient $\rho$.
The critic $Q_\psi$ used for proposal selection and refinement
denotes the mean of the critic ensemble.

We retain our expectile extension of QAM's squared TD-error loss:
$\ell_\kappa(e)
    =
    \left|\kappa-\mathbf{1}_{\{e<0\}}\right|e^2,$
where $e$ is the target minus the prediction.
The default $\kappa=0.5$ recovers squared-error regression up to
a constant factor; we use $\kappa=0.7$ on \texttt{c2} and
\texttt{c4}.
For an $h$-step backup, $r$ denotes the accumulated discounted
reward, $s'$ the bootstrap state, and $m$ the bootstrap mask.
The target therefore uses the discount $\gamma^h m$.
Losses below are averaged over the minibatch, and target
parameters are initialized as $\bar\psi\gets\psi$.

\begin{algorithm}[t]
\caption{PReFlow: Joint training update}
\label{algo:preflow-full}
\begin{algorithmic}
\Statex \textbf{Input:}
    Minibatch $\mathcal B$ of replay samples $(s,a,s',r,m)$,
    behavior flow $f_\beta$, refinement flow $v_\phi$,
    critic parameters $\psi$, target parameters $\bar\psi$,
    discount $\gamma$, backup horizon $h$, pessimism $\rho_Q$,
    target update rate $\lambda$,
    refinement scale $\delta$, inverse temperature $\tau$,
    and expectile $\kappa$.

\Statex \vspace{3pt}\textbf{Critic update}
\State Generate $\tilde a'$ at each $s'$ using
    Algorithm~\ref{algo:preflow-action} with $K=1$.
\State Clip $\tilde a'$ to the action bounds.
\State Update each critic $\psi^j$, $j=1,\ldots,10$,
    by minimizing
\Statex \hspace{\algorithmicindent}
    $\displaystyle
    \mathcal L(\psi^j)
    =
    \ell_\kappa\!\left(
        \operatorname{sg}\!\left[
            r+\gamma^h m\,\bar Q(s',\tilde a')
        \right]
        -Q_{\psi^j}(s,a)
    \right).
    $

\Statex \vspace{3pt}\textbf{Behavior-flow update}
\State Update $\beta$ using the flow-matching loss
    in~\eqref{eq:flow-matching} on $(s,a)$ from $\mathcal B$.

\Statex \vspace{3pt}\textbf{Refinement-flow update}
\State Update $\phi$ using Algorithm~\ref{algo:PReFlow}
    on the states in $\mathcal B$.
    \Comment{Training proposals use $K=1$}
\If{automatic temperature adaptation is enabled}
    \State Update $\log\tau$ as described in
        Appendix~\ref{app:autotau}, using the endpoints
        from the refinement update.
\EndIf

\Statex \vspace{3pt}\textbf{Target-network update}
\State $\bar\psi\gets(1-\lambda)\bar\psi+\lambda\psi$.

\Statex \vspace{3pt}\textbf{Output:}
    Updated $\beta,\phi,\psi,\bar\psi$, and $\tau$.
\end{algorithmic}
\end{algorithm}

\section{Experimental Details}
\label{app:experiments}

\subsection{Benchmark and evaluation protocol}
\label{app:protocol}

Our implementation builds on the official QAM codebase~\citep{qam},
and we follow its experimental protocol on the OGBench single-task
benchmark~\citep{ogbench}. Unless otherwise stated, we use QAM's
default architecture and optimization settings for the shared
components, including its domain-specific discount factors and critic
pessimism coefficients; see Appendices~E and~F of \citet{qam}.

We evaluate ten OGBench domains, with five tasks per domain
(50 tasks total):
\texttt{antmaze-large} (\texttt{al}),
\texttt{antmaze-giant} (\texttt{ag}),
\texttt{humanoidmaze-medium} (\texttt{hm}),
\texttt{humanoidmaze-large} (\texttt{hl}),
\texttt{scene},
\texttt{puzzle-3x3} (\texttt{p33}),
\texttt{puzzle-4x4} (\texttt{p44}),
\texttt{cube-double} (\texttt{c2}),
\texttt{cube-triple} (\texttt{c3}), and
\texttt{cube-quadruple} (\texttt{c4}).

Following \citet{qam}, we use the default \texttt{navigate} datasets
for \texttt{al}, \texttt{ag}, \texttt{hm}, and \texttt{hl}, and the
\texttt{play} datasets for the manipulation domains.
We use the 100M-transition datasets for \texttt{p44} and \texttt{c4},
and sparse rewards for \texttt{scene}, \texttt{p33}, and \texttt{p44}.
Action chunking uses $h=5$ for manipulation and $h=1$ for navigation.

The benchmark protocol consists of 1M offline
gradient steps followed by 500K online environment steps.
For our runs, we evaluate every 50K steps using 50 episodes per task
and seed, and report success rates at the end of each phase.
Domain scores average five tasks, and the overall score averages all
50 tasks.
For the main benchmark, PReFlow and our reproductions of TRQAM,
Q-Flow, and RQL use eight seeds per task.
Result sources and seed counts for the other baselines are given
in Appendix~\ref{app:baselines}.
We use three seeds for ablation studies and additional analyses,
unless otherwise stated.
The separate toy and ablation protocols are specified in
Appendices~\ref{app:toy-edit-family}, \ref{app:abl-edit-family},
and~\ref{app:selection-details}. 

\paragraph{Hardware.}
We ran our experiments on NVIDIA GPUs,
including RTX 3090, RTX A6000, RTX 4090, and RTX 5090 models.
Computational cost comparisons were measured
on a single NVIDIA RTX A6000 GPU,
with the timing protocol detailed in Appendix~\ref{app:timing}.

\subsection{PReFlow hyperparameters}
\label{app:hparams}

For our proposal-conditioned refinement velocity network, we use a four-layer MLP with 512 hidden
units per layer, where the state and proposal are concatenated as input. Both the behavior and refinement flows use 10 Euler
integration steps.

Table~\ref{tab:app-domain} reports the per-domain PReFlow settings. We search over fixed inverse temperatures $\tau\in\{3,10,30,100\}$ and dispersion targets $c\in\{1.5,2.0,2.5,5.0,15.0\}$ for automatic temperature adaptation (Appendix~\ref{app:autotau}). We also search over refinement scales $\delta\in\{0.05,0.1,0.2,0.5,0.9\}$ and numbers of behavior proposals $K\in\{1,2,4,8\}$. We evaluate each candidate configuration with three seeds per task, and select the configuration that maximizes the sum of the success rates at the end of offline training and online fine-tuning. We then fix this configuration for all five tasks in the domain and report results averaged over eight seeds per task.

We use QAM's squared TD-error objective~\citep{qam} as the default for critic learning. Following RQL~\citep{rql}, we also consider an asymmetric version that gives different weights to predictions below and above their TD targets. This expectile regression loss is controlled by $\kappa$: our default $\kappa=0.5$ recovers QAM's loss up to a constant factor, while $\kappa>0.5$ gives more weight to predictions below their targets. We use $\kappa=0.7$ only on \texttt{c2} and \texttt{c4}.

TRQAM~\citep{dong2026trustregionqadjoint} increases its KL budget at the offline-to-online transition on \texttt{ag}. Motivated by this choice, we test increasing PReFlow's refinement scale $\delta$ from $0.1$ during offline training to $0.2$ during online fine-tuning. We apply this adjustment only to \texttt{ag}. Both this adjustment and the expectile choices above are treated as domain-specific hyperparameter settings.

\begin{table}[t]
\centering
\begingroup
\footnotesize
\renewcommand{\arraystretch}{1.2}
\setlength{\tabcolsep}{3pt}
\begin{tabular}{@{}lcccl@{}}
\toprule
\texttt{\textbf{Domain}} & $\boldsymbol{\tau}$ & $\boldsymbol{\delta}$
& $\boldsymbol{K}$ & \texttt{\textbf{Notes}} \\
\midrule
\texttt{\textbf{al}}      & auto ($c{=}2.5$) & 0.1  & 1 & -- \\
\texttt{\textbf{ag}}      & auto ($c{=}1.5$) & 0.1  & 1
& $\delta_{\mathrm{online}}{=}0.2$ \\
\texttt{\textbf{hm}}   & 30               & 0.1  & 4 & -- \\
\texttt{\textbf{hl}} & 100              & 0.05 & 4 & -- \\
\texttt{\textbf{scene}}              & 3                & 0.1  & 4 & sparse \\
\texttt{\textbf{p33}}         & 3                & 0.1  & 4 & sparse \\
\texttt{\textbf{p44}}         & auto ($c{=}2.0$) & 0.9  & 1 & sparse, $100M$ \\
\texttt{\textbf{c2}}        & 3                & 0.1  & 8 & $\kappa=0.7$ \\
\texttt{\textbf{c3}}        & 30               & 0.1  & 1 & -- \\
\texttt{\textbf{c4}}     & auto ($c{=}15$)  & 0.1  & 4 & $\kappa=0.7$, $100M$ \\
\bottomrule
\end{tabular}
\endgroup
\caption{
\textbf{Per-domain PReFlow settings.}
``auto ($c$)'' denotes the temperature adaptation of Appendix~\ref{app:autotau}
with target $c$.
}
\label{tab:app-domain}
\end{table}

\paragraph{Sensitivity sweep.}
For Figure~\ref{fig:sensitivity}, we vary $\delta$ and fixed
$\tau$ with $K=1$ and $\kappa=0.5$, without temperature adaptation.
We evaluate one task per domain, using task~1 for navigation
and task~2 for manipulation.
For \texttt{c4}, all tested configurations yielded zero success
on task~2, so we instead use task~1 to illustrate performance
differences across hyperparameter settings.
Each configuration is trained for 1M offline gradient steps
without online fine-tuning.
Each cell reports the mean success rate over three seeds
at the 1M-step checkpoint.

\subsection{Baselines}
\label{app:baselines}

For ReBRAC, FBRAC, BAM, FQL, FAWAC, CGQL and its variants, DAC, QSM,
DSRL, FEdit, IFQL, and QAM and its variants, we use the released results
from QAM~\citep{qam}, with 12 seeds per task.
Table~\ref{tab:main_combined} reports a subset of these methods;
the full task-level tables additionally include FBRAC, BAM, FQL,
FAWAC, CGQL-M, and CGQL-L.
We refer to Appendices~E and~F of \citet{qam} for their implementations
and hyperparameter settings.
We re-run TRQAM, Q-Flow, and RQL using the protocol in
Appendix~\ref{app:protocol}, with eight seeds per task.
For QPILOTS-M and QPILOTS-U, we use the authors' released results
with eight seeds per task, as detailed below.

\paragraph{TRQAM.}
TRQAM~\citep{dong2026trustregionqadjoint} extends QAM with a path-space
KL constraint relative to the behavior flow. It adjusts a trust-region
parameter in the sampling dynamics through projected dual updates,
using the estimated KL along sampled trajectories to track a prescribed
budget $\epsilon_{\mathrm{KL}}$.

We use the authors' reported budget for $\epsilon_{\mathrm{KL}}$
except \texttt{hm} and \texttt{hl}. Using the published budget of $\epsilon_{\mathrm{KL}}=0.5$ left the KL constraint unsatisfied under our protocol, with the dual variable
reaching its optimization boundary. 

Following the authors' tuning protocol, we
sweep $\epsilon_{\mathrm{KL}}\in\{1.0,2.0,4.0\}$ on tasks~1 and~4
using two seeds per configuration.
We select $\epsilon_{\mathrm{KL}}=1.0$ for both domains and fix it
across all five tasks. We do not use BC pre-training.

\paragraph{Q-Flow.}
Q-Flow~\citep{qflow} learns a value function over intermediate flow
states. For a sampled intermediate state, it runs the policy to the
endpoint and uses the terminal critic value as a regression target.
The policy is then trained by adding the gradient of this intermediate
value function to the standard flow-matching target, avoiding
backpropagation through the full flow rollout.

We use the authors' released implementation.
The original paper averages offline evaluations at 800K, 900K,
and 1M gradient steps. We instead report the evaluation at 1M steps
to match the other methods in our comparison.

\paragraph{RQL.}
RQL~\citep{rql} treats individual flow-refinement steps as actions in
an expanded MDP. It reverses flows from dataset actions to construct
virtual trajectories for off-policy learning, and reduces the effective value-learning horizon introduced by the refinement steps.
This allows it to train the full flow policy without backpropagating
through the entire sampling process.

We extend the implementation to support online data collection and the
streaming OGBench datasets, and use the authors' reported per-domain
hyperparameters.

\paragraph{QPILOTS.}
QPILOTS~\citep{qpilots} steers a behavior flow at inference time by
computing critic guidance from clean-action estimates at each flow step.
QPILOTS-U uses a Tweedie point estimate, while QPILOTS-M uses approximate
samples from the conditional endpoint distribution generated by an
auxiliary meta flow map.

We use the authors' released results for both variants, based on eight
seeds per task and the hyperparameter settings reported in their paper.

\subsection{Toy example: conditional multimodality}
\label{app:toy-edit-family}

This example illustrates the conditional refinement policies in
Section~\ref{sec:refinement} and Figure~\ref{fig:toy_qvpo}.
We adapt the three-peak continuous bandit of QVPO~\citep{qvpo}
and use its reward as a fixed critic, omitting the state argument:
\[
Q(\tilde a)=\sum_{i=1}^{3}\frac{w}{2\pi\sigma^2}
\exp\left(-\frac{\|\tilde a-\mu_i\|^2}{2\sigma^2}\right).
\]
We use $w=1.5$, $\mu_1=(-1.35,0.65)$, $\mu_2=(-0.65,1.35)$,
$\mu_3=(-1.61,1.61)$, and $\sigma=0.2$ in place of the original $0.1$.
We fix the proposal $a=(-1.40,1.40)$, the refinement scale
$\delta=0.5$, and the inverse temperature $\tau=2$.
Writing $\Delta a=\delta y$, the conditional objective
in~\eqref{eq:conditional-objective} gives the target
\[
q^\star(y)\propto \N(y;0,I)\,e^{\tau Q(a+\delta y)}.
\]

We compare flow refinement, full-covariance Gaussian refinement,
and QAM~\citep{qam} on the same problem.
The two refinement policies use the same conditional objective,
critic, proposal, reference, $\delta$, and $\tau$.
All three methods are trained for 100K steps using Adam with
global-norm gradient clipping at $1.0$.
The initial learning rate is $\eta_0=5\times10^{-4}$,
reduced to $\eta_0/10$ at 60\% of training
and $\eta_0/30$ at 90\%.
The flow and Gaussian use batch sizes of $4096$ and $32{,}768$,
respectively.
The flow uses a $3\times256$ MLP with a zero-initialized output
layer and 32 Euler integration steps.
The Gaussian is trained using reparameterization gradients
of the same conditional objective.
For QAM, the behavior policy is set to the reference law
$\pi_\beta=\N(a,\delta^2 I)$.
We realize it as a rectified flow from $\N(0,I)$
with a closed form velocity field.
QAM fine-tunes this flow in action space by adjoint matching
using~\eqref{eq:memoryless-sde}-\eqref{eq:adjoint matching}
with the same critic and $\tau$,
and uses the same batch size, network, and number of
discretization steps as the flow refinement.
Its target density is proportional to
$\pi_\beta(\tilde a)e^{\tau Q(\tilde a)}$
and corresponds to $q^\star$ under $\tilde a=a+\delta y$. 

Figure~\ref{fig:toy-kl} tracks the KL divergence from each learned
distribution $q$ to the target during training,
$D_{\mathrm{KL}}(q\,\|\,q^\star)
=\mathbb E_{y\sim q}\bigl[\log q(y)-\log q^\star(y)\bigr]$,
which we estimate with 100{,}000 samples from $q$.
Flow refinement attains the lowest mean KL at every evaluated step and ends below $0.1$. 

\begin{figure}[t]
    \centering
    \includegraphics[width=0.6\linewidth]{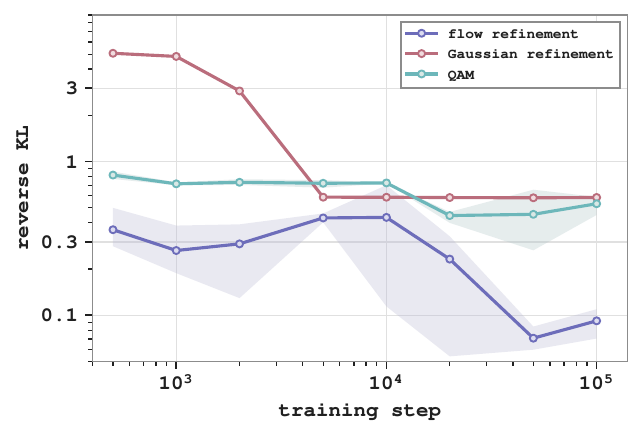}
    \captionsetup{font=small,skip=4pt}
    \caption{
    \textbf{KL divergence to the target on the toy problem.}
    Lines show means over 3 seeds, and shaded regions show
    one standard deviation. Both axes are logarithmic.
    }
    \label{fig:toy-kl}
\end{figure}

\subsection{Refinement policy class on OGBench}
\label{app:abl-edit-family}

We describe the refinement policy comparison in
Table~\ref{tab:abl-edit-family} of Sec.~\ref{sec:exp-refinement}.
We use a learned diagonal Gaussian policy,
$q_\phi(y\mid s,a)=\mathcal N(\mu_\phi(s,a),
\operatorname{diag}(\sigma_\phi^2(s,a)))$.
The learned Gaussian uses the same state and proposal inputs and MLP
hidden-layer sizes as the refinement flow, and is trained to maximize
\[
\tau\,\mathbb E_{y\sim q_\phi(\cdot\mid s,a)}
\left[Q_\psi(s,a+\delta y)\right]
-\mathrm{KL}\!\left(
q_\phi(\cdot\mid s,a)\|\mathcal N(0,I)
\right).
\] 
We also include a variant without refinement ($\delta=0$).
The learned Gaussian and flow both use $\delta=0.1$; the
no-refinement variant returns the behavior proposal directly. 
All variants use $K=1$.
We evaluate on \texttt{al}, \texttt{hm}, \texttt{scene}, and
\texttt{c2}, with five tasks each and three seeds per task. Both refiners use fixed inverse temperatures of $100$, $30$, $3$,
and $3$ on \texttt{al}, \texttt{hm}, \texttt{scene}, and
\texttt{c2}, respectively. We use $\kappa=0.7$ on \texttt{c2} and $\kappa=0.5$ elsewhere. We report offline success rates after 1M gradient steps, evaluated
over 50 episodes per task and seed.

\subsection{Proposal selection}
\label{app:selection-details}

We compare three rules for selecting among $K$ behavior proposals:
uniform sampling, sampling with probabilities proportional to
$\exp(\tau Q_\psi(s,a_i))$, and selecting the
proposal with the highest critic value.
The selected proposal is then refined using the learned refinement flow.
For softmax selection, we use the domain-specific $\tau$ values in
Table~\ref{tab:app-domain}.
Each setting is evaluated over three seeds, with each seed
covering all five tasks in a domain, as reported in
Table~\ref{tab:abl-selection}.

\subsection{Timing methodology}
\label{app:timing}

Figure~\ref{fig:flowstep-cost} reports mean full-training-update
times on a single NVIDIA RTX A6000 GPU, averaged over five flow-step
counts: 5, 10, 20, 50, and 100.

Table~\ref{tab:speed} provides a separate timing breakdown at 10
Euler steps. We use
\texttt{c2} input shapes: 37-dimensional observations and
5-dimensional actions with chunk length $h=5$ (25-dimensional action
chunks), batch size 256, and 10 Euler steps per flow rollout.
Each method runs in a separate subprocess. After five warm-up calls,
we time 30 calls with device synchronization via
\texttt{jax.block\_until\_ready} and report the median in milliseconds.

\section{Additional Experimental Results}
\label{app:results} 


\subsection{Comparison with SPAR}
\label{app:spar}

SPAR~\citep{spar} also refines an anchor action with a residual policy.
It trains a conditional VAE over the residual by advantage-weighted
reconstruction and latent self-imitation. We port the authors'
implementation to the QAM codebase and keep SPAR's own critic, anchor,
and action-selection rule. 

SPAR first trains an IQL critic and a
tanh-Gaussian anchor for 1M gradient steps, and then trains the
residual for 1M steps with both networks frozen. At evaluation, it
executes the best of ten residual candidates only when their robust
critic value exceeds that of the anchor. Network sizes, action
chunking, discount factors, and the evaluation protocol match PReFlow.
We tune the anchor (Gaussian or behavior flow), the advantage filter
(hard or soft), the guide weight $\lambda_g\in\{0.5,1,2\}$, and the
ensemble penalty $\lambda_u\in\{0,0.5,1\}$ on \texttt{al}, \texttt{hm},
\texttt{scene}, and \texttt{c2} with three seeds per task. 

The paper's
main configuration (Gaussian anchor, hard filter, $\lambda_g=1$,
$\lambda_u=0.5$) gives the highest average, so we run it with eight
seeds per task. Table~\ref{tab:spar} reports offline success rates
after 1M gradient steps of residual training. SPAR improves on its
anchor on \texttt{al} and \texttt{scene}, but it remains below PReFlow
in all four domains. 

\begin{table}
\centering
\begingroup

\def\SPARbest#1{{%
  \setlength{\fboxsep}{0.6pt}%
  \colorbox{background}{\strut #1}}}
\def\SPARstat#1#2{%
  \shortstack[c]{%
    {\fontsize{5.5}{6.2}\selectfont\textcolor{gray}{[#2]}}%
    \\[-1pt]\vphantom{\SPARbest{0}}#1}}

\fontsize{7.5}{9.0}\selectfont
\setlength{\tabcolsep}{8pt}
\renewcommand{\arraystretch}{1.2}

\begin{tabular}{@{}lcccc@{}}
\toprule
& \texttt{\textbf{al}}
& \texttt{\textbf{hm}}
& \texttt{\textbf{scene}}
& \texttt{\textbf{c2}} \\
Method
& {\color{gray}\fontsize{5.5}{6.2}\selectfont 5 tasks}
& {\color{gray}\fontsize{5.5}{6.2}\selectfont 5 tasks}
& {\color{gray}\fontsize{5.5}{6.2}\selectfont 5 tasks}
& {\color{gray}\fontsize{5.5}{6.2}\selectfont 5 tasks} \\
\midrule
\texttt{\textbf{SPAR anchor}}
& \SPARstat{48}{43,52}
& \SPARstat{7}{5,8}
& \SPARstat{14}{12,15}
& \SPARstat{3}{2,4} \\
\texttt{\textbf{SPAR}}
& \SPARstat{60}{55,64}
& \SPARstat{8}{6,9}
& \SPARstat{27}{26,29}
& \SPARstat{3}{2,4} \\
\texttt{\textbf{PReFlow}}
& \SPARstat{\SPARbest{82}}{79,85}
& \SPARstat{\SPARbest{65}}{60,69}
& \SPARstat{\SPARbest{97}}{96,97}
& \SPARstat{\SPARbest{59}}{56,62} \\
\bottomrule
\end{tabular}
\endgroup

\caption{
\textbf{Comparison with SPAR.}
Offline success rates (\%) averaged over eight seeds and five tasks
per domain, with 95\% confidence intervals in small gray brackets
above the means. Best results per domain are highlighted.
SPAR is evaluated after 1M gradient steps of critic and anchor
training followed by 1M steps of residual training.
}
\label{tab:spar}
\end{table}

\subsection{Computational cost}
\label{app:timing-results}

Table~\ref{tab:speed} provides the timing breakdown at 10 Euler steps.
The timing methodology is described in Appendix~\ref{app:timing}.

\begin{table}
    \centering
    \begingroup
    \footnotesize
    \renewcommand{\arraystretch}{1.2}
    \setlength{\tabcolsep}{3pt}
    \begin{tabular}{@{}lrrr@{}}
    \toprule
    Method & \texttt{\textbf{Actor grad.}}
    & \texttt{\textbf{Critic grad.}} & \texttt{\textbf{Full step}} \\
    \midrule
    \texttt{\textbf{QAM}}         &  3.40 & \cellbg 2.13 &  8.22 \\
    \texttt{\textbf{QAM-E}}       &  5.08 &  2.71 &  9.97 \\
    \texttt{\textbf{PReFlow}} (ours) &  2.65 &  2.43 & \cellbg 7.18 \\
    \texttt{\textbf{QPILOTS-U}}   & \cellbg 0.84 &  6.99 & 10.93 \\
    \texttt{\textbf{QPILOTS-M}}   &  0.91 & 49.69 & 51.20 \\
    \bottomrule
    \end{tabular}
    \endgroup
    \caption{
\textbf{Training-update cost on \texttt{c2} (ms).}
Median times over 30 calls after warm-up, measured on the same GPU.
Actor and critic columns measure loss-gradient computation without
optimizer updates; the full step includes all training updates.
\best{The lowest time in each column is highlighted}.
}
    \label{tab:speed}
\end{table}

\subsection{Effect of Proposal Selection}
\label{app:abl-selection}

Figure~\ref{fig:selection_ablation_full} provides the full learning
curves for the experiment in Section~\ref{sec:exp-selection} across
all ten OGBench domains.  

Table~\ref{tab:abl-selection} shows that critic-guided selection improves
performance over uniform selection on \texttt{c2} and
\texttt{scene}, while the improvement on \texttt{hm}
is limited.
Argmax achieves the highest mean success rate in four of the six
domain-$K$ settings and remains competitive in the other two.
Based on these results, we use argmax selection in our benchmark
experiments.
Appendix~\ref{app:proposal-selection} distinguishes greedy selection
from the finite-temperature policy derived from the joint objective.

\begin{table}[!htbp]
\centering
\begingroup
\footnotesize
\def\PReFlowSelectionStat#1#2{#1{\color{gray}\scriptsize$\pm$#2}}
\renewcommand{\arraystretch}{1.2}
\setlength{\tabcolsep}{3pt}
\begin{tabular}{@{}lccccccc@{}}
\toprule
& & \multicolumn{3}{c}{\color{gray}$\boldsymbol{K=4}$}
& \multicolumn{3}{c}{\color{gray}$\boldsymbol{K=8}$} \\
\cmidrule(lr){3-5}
\cmidrule(lr){6-8}
\texttt{\textbf{Domain}} & $\boldsymbol{K=1}$
& \texttt{\textbf{Uniform}} & \texttt{\textbf{Softmax}} & \texttt{\textbf{Argmax}}
& \texttt{\textbf{Uniform}} & \texttt{\textbf{Softmax}} & \texttt{\textbf{Argmax}} \\
\midrule
\texttt{\textbf{hm}} {\color{gray}\scriptsize($\tau=30$)}
& \PReFlowSelectionStat{69.2}{1.2}
& \PReFlowSelectionStat{68.8}{2.6}
& \cellbg\PReFlowSelectionStat{70.1}{3.1}
& \PReFlowSelectionStat{70.0}{2.9}
& \PReFlowSelectionStat{68.7}{0.8}
& \PReFlowSelectionStat{66.0}{3.3}
& \cellbg\PReFlowSelectionStat{69.2}{3.4} \\
\texttt{\textbf{scene}} {\color{gray}\scriptsize($\tau=3$)}
& \PReFlowSelectionStat{73.1}{0.8}
& \PReFlowSelectionStat{74.0}{1.7}
& \PReFlowSelectionStat{89.5}{0.8}
& \cellbg\PReFlowSelectionStat{94.9}{1.0}
& \PReFlowSelectionStat{74.1}{0.8}
& \PReFlowSelectionStat{90.7}{1.7}
& \cellbg\PReFlowSelectionStat{92.8}{0.6} \\
\texttt{\textbf{c2}} {\color{gray}\scriptsize($\tau=3$)}
& \PReFlowSelectionStat{53.9}{3.3}
& \PReFlowSelectionStat{50.3}{6.9}
& \PReFlowSelectionStat{61.3}{1.0}
& \cellbg\PReFlowSelectionStat{63.3}{4.2}
& \PReFlowSelectionStat{49.6}{8.0}
& \cellbg\PReFlowSelectionStat{60.1}{3.9}
& \PReFlowSelectionStat{58.7}{4.1} \\
\bottomrule
\end{tabular}
\endgroup
\caption{
\textbf{Effect of the selection rule.}
Offline success rates (\%) after 1M gradient steps.
Entries report mean $\pm$ standard deviation over three seeds,
with each seed's score averaged over five tasks.
The listed $\tau$ values apply to softmax selection.
\best{Best results are highlighted} among the three rules for each domain and $K$.
}
\label{tab:abl-selection}
\end{table}

\subsection{Proposal selection without refinement}
\label{app:abl-norefine}

To separate the contributions of proposal selection and refinement,
we repeat the $K$ sweep of Section~\ref{sec:exp-selection} with
refinement disabled ($\delta=0$), so that the selected behavior
proposal is executed directly. All other settings follow
Section~\ref{sec:exp-selection}.
Figure~\ref{fig:abl-norefine} reports offline success rates averaged
over four domains (\texttt{al}, \texttt{hm}, \texttt{scene}, \texttt{c2}).
Without refinement, $K=1$ reduces to the behavior policy and scores
near zero. Selecting among more proposals helps, but remains far below refinement with a single proposal.
Refinement therefore accounts for most of the improvement, while
proposal selection complements it.

\begin{figure}[t]
    \centering
    \includegraphics[width=\linewidth]{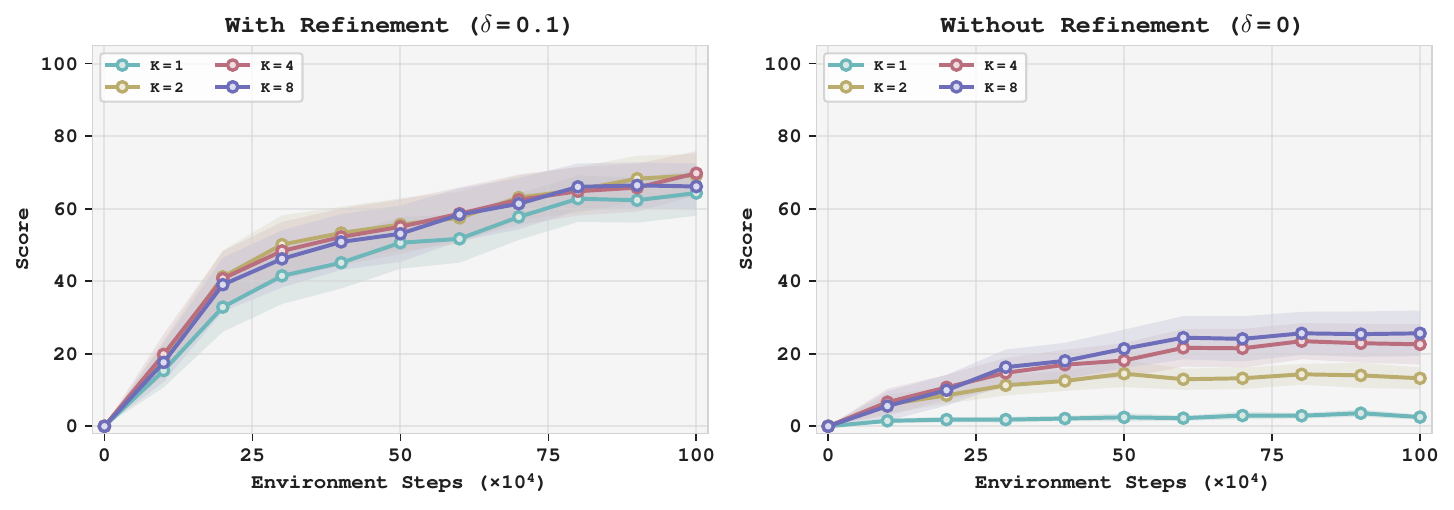}
    \captionsetup{font=small,skip=4pt}
    \caption{
    \textbf{Proposal selection with and without refinement.}
    Offline success rates average on \texttt{c2, hm, scene, al} for
    $K\in\{1,2,4,8\}$, with refinement ($\delta=0.1$, left) and
    without refinement ($\delta=0$, right).
    Without refinement, the selected proposal is executed directly. Curves show means over three seeds; shaded regions show pointwise
95\% confidence intervals.
    }
    \label{fig:abl-norefine}
\end{figure}

\begin{figure}[t]
    \centering
    \includegraphics[width=\linewidth]{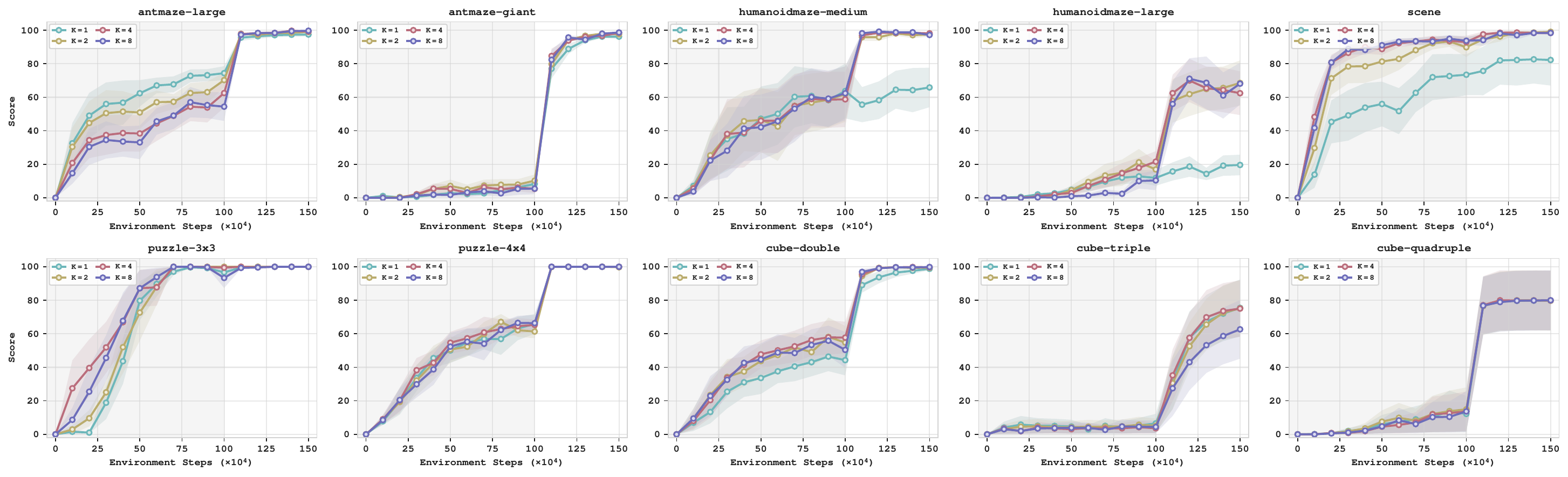}
    \captionsetup{font=small,skip=4pt}
    \caption{
    \textbf{Effect of proposal selection across all ten OGBench domains.}
    PReFlow with $K\in\{1,2,4,8\}$ behavior proposals.
    All settings use the learned refinement flow to refine the proposal
    with the highest critic value.
    Curves show means over three seeds; shaded regions show pointwise
95\% confidence intervals. 
    }
    \label{fig:selection_ablation_full}
\end{figure}

\subsection{Domain-specific hyperparameter choices}
\label{app:hparam-results}

The hyperparameter settings are given in Appendix~\ref{app:hparams}.

\paragraph{Critic expectile.}

With all other hyperparameters fixed, increasing $\kappa$ from $0.5$ to $0.7$ improves offline performance by 4.9 and 5.9 percentage points on \texttt{c2} and \texttt{c4}, respectively. This change does not consistently improve performance on the other domains. We hypothesize that increased critic overestimation limits its benefit on those domains.

\paragraph{Refinement scale during online fine-tuning.}

Increasing the refinement scale $\delta$ from $0.1$ during offline
training to $0.2$ during online fine-tuning improves online performance
on \texttt{ag}, but does not yield significant gains on other domains.

\subsection{Full task-level results}
\label{app:full-results}

Tables~\ref{tab:task_offline} and~\ref{tab:task_online} report success
rates on all 50 tasks at the end of offline training and after online
fine-tuning, respectively. PReFlow, TRQAM, Q-Flow, and RQL use eight
seeds per task under our common protocol. For the baselines
evaluated by \citet{qam}, we use their released results with 12 seeds
per task. 
Each task cell reports the mean success rate with a 95\% confidence
interval over seeds. Aggregate rows report the mean across the five
tasks in each domain. Figure~\ref{fig:full_training_curves} provides
the task-level learning curves.
QPILOTS-M/U are included in the domain-level comparison in
Table~\ref{tab:main_combined}, but not in these task-level tables;
their result sources are given in Appendix~\ref{app:baselines}.  

\subsection{Full training curves}
\label{app:full-curves}

Figure~\ref{fig:full_training_curves} presents the training curves of
PReFlow on all 50 tasks, together with the curves of the baselines
evaluated by \citet{qam}, plotted from their released results.
Online fine-tuning begins after 1M offline gradient steps and continues
for 500K environment steps.

\providecolor{ourblue}{RGB}{0,94,158}
\providecolor{oursbg}{RGB}{232,246,250}

\newcommand{\PReFlowTaskMean}[2]{%
  \def\PReFlowtmp{#2}\def\PReFlowdash{-}%
  \ifx\PReFlowtmp\PReFlowdash
    \text{--}%
  \else
    \ifnum#1=#2\relax
      \textcolor{ourblue}{\textbf{#2}}%
    \else
      #2%
    \fi
  \fi
}
\def\PReFlowTaskCell#1#2/#3/#4\PReFlowCellStop{%
  \if\relax\detokenize{#3}\relax
    \PReFlowTaskMean{#1}{#2}%
  \else
    $\overset{\text{{\fontsize{3.6}{4}\selectfont$[#3,#4]$}}}%
      {\text{\PReFlowTaskMean{#1}{#2}}}$%
  \fi
}
\def\PReFlowTaskCells#1#2,#3\PReFlowListStop{%
  &\if\relax\detokenize{#3}\relax
    \cellcolor{oursbg}\PReFlowTaskCell{#1}#2\PReFlowCellStop
  \else
    \PReFlowTaskCell{#1}#2\PReFlowCellStop
    \PReFlowTaskCells{#1}#3\PReFlowListStop
  \fi
}
\newcommand{\PReFlowTaskRow}[3]{%
  &\texttt{#1}\PReFlowTaskCells{#2}#3,\PReFlowListStop\\
}
\newcommand{\PReFlowTaskAgg}[3]{%
  \multirow{-6}{*}{\texttt{#1}}&\texttt{agg. (5 tasks)}%
  \PReFlowTaskCells{#2}#3,\PReFlowListStop\\
}

\begin{table}[p]
\centering
\begingroup
\newcommand{\PReFlowTaskBest}[1]{{%
  \setlength{\fboxsep}{0.4pt}%
  \colorbox{background}{#1}}}
\renewcommand{\PReFlowTaskMean}[2]{%
  \def\PReFlowtmp{#2}\def\PReFlowdash{-}%
  \ifx\PReFlowtmp\PReFlowdash
    \text{--}%
  \else
    \ifnum#1=#2\relax
      \PReFlowTaskBest{#2}%
    \else
      #2%
    \fi
  \fi
}
\def\PReFlowTaskCell#1#2/#3/#4\PReFlowCellStop{%
  \if\relax\detokenize{#3}\relax
    \vphantom{\PReFlowTaskBest{0}}\PReFlowTaskMean{#1}{#2}%
  \else
    \shortstack[c]{%
      {\fontsize{5.3}{5.8}\selectfont\textcolor{gray}{[#3,#4]}}%
      \\[-1.4pt]\vphantom{\PReFlowTaskBest{0}}\PReFlowTaskMean{#1}{#2}}%
  \fi
}
\def\PReFlowTaskCells#1#2,#3\PReFlowListStop{%
  &\PReFlowTaskCell{#1}#2\PReFlowCellStop
  \if\relax\detokenize{#3}\relax
  \else
    \PReFlowTaskCells{#1}#3\PReFlowListStop
  \fi
}
\renewcommand{\PReFlowTaskRow}[3]{%
  \texttt{\textbf{#1}}\PReFlowTaskCells{#2}#3,\PReFlowListStop\\
}
\renewcommand{\PReFlowTaskAgg}[3]{%
  \texttt{\textbf{agg.}}\PReFlowTaskCells{#2}#3,\PReFlowListStop\\
}
\newcommand{\PReFlowTaskGroup}[1]{%
  \addlinespace[2pt]
  \multicolumn{21}{@{}l@{}}{%
    {\color{gray}\ttfamily\bfseries #1}}\\[-1pt]
}
\newcommand{\PReFlowTaskHead}[1]{%
  {\fontsize{7.0}{8.0}\selectfont\ttfamily\bfseries #1}%
}
\fontsize{8.0}{9.0}\selectfont
\setlength{\tabcolsep}{5.5pt}
\renewcommand{\arraystretch}{1.0}
\maxsizebox*{\linewidth}{0.86\textheight}{%
\begin{tabular}{@{}l*{20}{c}@{}}
\toprule
\PReFlowTaskHead{Task} & \PReFlowTaskHead{ReBRAC} & \PReFlowTaskHead{FBRAC} & \PReFlowTaskHead{BAM} & \PReFlowTaskHead{FQL} & \PReFlowTaskHead{FAWAC} & \PReFlowTaskHead{CGQL} & \PReFlowTaskHead{CGQL-M} & \PReFlowTaskHead{CGQL-L} & \PReFlowTaskHead{DAC} & \PReFlowTaskHead{QSM} & \PReFlowTaskHead{DSRL} & \PReFlowTaskHead{FEdit} & \PReFlowTaskHead{IFQL} & \PReFlowTaskHead{QAM} & \PReFlowTaskHead{QAM-F} & \PReFlowTaskHead{QAM-E} & \PReFlowTaskHead{TRQAM} & \PReFlowTaskHead{Q-Flow} & \PReFlowTaskHead{RQL} & \PReFlowTaskHead{PReFlow} \\
\midrule
\PReFlowTaskGroup{al}
\PReFlowTaskRow{task1}{98}{98/97/99,0/0/0,90/88/93,93/89/96,6/3/9,63/49/76,56/49/63,39/31/47,88/83/92,87/79/94,62/53/71,67/60/73,41/27/54,77/67/85,85/80/89,87/83/91,85/80/88,96/94/98,80/77/84,79/67/90}
\PReFlowTaskRow{task2}{92}{88/85/91,0/0/0,60/54/65,86/83/89,1/0/2,73/64/79,49/45/54,51/47/55,71/67/75,78/71/85,75/68/82,67/64/70,14/10/19,80/76/83,64/59/68,81/78/85,62/52/72,92/89/95,78/73/82,78/71/83}
\PReFlowTaskRow{task3}{99}{98/97/99,12/5/20,94/92/96,59/48/68,40/32/48,91/89/94,90/87/93,79/75/82,98/97/99,99/98/100,81/77/85,65/52/77,54/44/63,89/86/93,96/93/97,94/92/96,93/92/94,98/96/100,94/91/97,90/85/95}
\PReFlowTaskRow{task4}{94}{94/92/95,0/0/0,85/83/88,54/41/68,18/15/20,78/74/83,78/75/81,76/73/80,91/88/94,91/88/94,18/5/32,28/20/36,25/20/29,69/52/81,81/76/86,70/63/76,80/76/85,92/90/94,80/74/84,79/74/83}
\PReFlowTaskRow{task5}{96}{95/94/96,0/0/0,88/86/91,86/83/88,22/16/28,75/69/81,80/76/84,78/75/81,92/90/94,96/94/97,69/65/73,63/56/69,45/33/55,89/86/91,87/86/89,83/79/87,91/86/94,96/94/98,76/73/80,82/78/87}
\PReFlowTaskAgg{al}{95}{94/94/95,2/1/4,84/82/85,76/72/79,17/15/19,76/73/80,71/68/73,65/62/67,88/86/90,90/88/92,61/56/66,58/54/62,36/32/39,81/78/84,83/81/84,83/80/86,82/80/85,95/94/96,82/80/84,82/79/85}
\PReFlowTaskGroup{ag}
\PReFlowTaskRow{task1}{62}{36/31/41,0/0/0,3/0/9,0/0/0,0/0/0,2/0/8,0/0/0,0/0/0,53/45/61,62/49/73,0/0/0,0/0/0,0/0/0,7/5/9,11/8/15,0/0/0,2/0/3,11/6/18,15/8/22,5/3/7}
\PReFlowTaskRow{task2}{73}{73/68/78,0/0/0,0/0/0,0/0/0,0/0/0,0/0/0,0/0/0,4/2/6,0/0/0,1/0/2,0/0/0,0/0/0,0/0/0,0/0/0,0/0/0,0/0/0,8/4/10,42/29/55,56/38/72,22/15/29}
\PReFlowTaskRow{task3}{47}{13/9/18,0/0/0,0/0/0,0/0/0,0/0/0,0/0/0,0/0/0,0/0/0,0/0/0,0/0/0,0/0/0,0/0/0,1/0/2,0/0/1,0/0/0,0/0/0,0/0/1,4/1/6,47/38/56,0/0/1}
\PReFlowTaskRow{task4}{74}{74/59/84,0/0/0,0/0/0,0/0/0,0/0/0,0/0/0,0/0/0,0/0/0,0/0/0,0/0/1,12/5/19,10/3/17,2/1/4,34/25/43,5/0/15,0/0/0,7/4/9,50/44/58,23/15/32,7/2/14}
\PReFlowTaskRow{task5}{89}{89/85/92,0/0/0,1/0/3,0/0/0,0/0/0,0/0/0,22/6/40,11/0/26,25/6/47,58/36/77,2/1/3,0/0/0,2/0/5,49/29/68,43/25/61,6/0/16,52/45/58,72/68/75,63/57/68,24/10/40}
\PReFlowTaskAgg{ag}{57}{57/53/60,0/0/0,1/0/2,0/0/0,0/0/0,0/0/2,4/1/8,3/1/6,16/11/20,24/19/29,3/1/4,2/1/3,1/0/2,18/14/22,12/8/16,1/0/3,14/12/15,36/32/39,41/36/45,12/8/15}
\PReFlowTaskGroup{hm}
\PReFlowTaskRow{task1}{99}{38/27/50,26/21/30,49/47/52,34/21/49,18/16/21,30/27/34,8/1/17,55/52/59,87/77/93,88/83/92,49/32/64,0/0/0,86/84/88,40/30/49,32/20/46,27/13/43,15/11/20,90/87/92,99/98/100,59/45/70}
\PReFlowTaskRow{task2}{100}{91/80/98,78/75/82,69/62/76,95/91/98,44/40/48,78/75/82,99/98/100,93/91/95,96/94/97,96/94/98,91/88/93,39/30/48,92/90/95,97/96/99,98/97/99,99/98/100,38/31/44,96/95/98,100/99/100,98/96/100}
\PReFlowTaskRow{task3}{100}{83/62/98,28/20/35,75/72/79,96/95/98,20/18/23,78/65/87,0/0/1,62/38/85,92/85/96,95/92/97,36/22/50,0/0/0,93/91/95,96/93/98,95/94/97,68/53/80,28/24/30,95/93/97,100/100/100,68/51/82}
\PReFlowTaskRow{task4}{84}{37/20/54,3/1/5,22/18/25,14/4/26,1/0/2,23/19/28,0/0/0,2/1/3,43/34/52,31/23/39,0/0/0,0/0/0,60/55/65,3/0/6,0/0/0,0/0/0,9/6/11,56/48/66,84/64/95,0/0/0}
\PReFlowTaskRow{task5}{100}{96/94/98,59/49/68,83/79/86,99/98/100,36/32/42,89/88/91,100/99/100,98/97/100,99/98/99,98/97/99,90/86/93,68/61/74,98/96/99,99/98/100,99/98/100,99/98/100,51/44/56,99/98/100,99/98/100,99/99/100}
\PReFlowTaskAgg{hm}{96}{69/65/74,39/37/41,60/58/62,68/63/73,24/22/26,60/57/62,42/40/43,62/57/68,83/81/85,82/79/84,53/48/58,22/20/23,86/85/87,67/64/69,65/62/68,59/54/63,28/26/30,87/86/89,96/92/99,65/60/69}
\PReFlowTaskGroup{hl}
\PReFlowTaskRow{task1}{76}{36/29/43,0/0/0,5/1/9,8/5/11,0/0/0,2/1/4,0/0/1,2/1/4,0/0/1,10/7/12,14/8/20,8/6/9,36/31/41,6/3/9,8/4/12,6/2/13,0/0/0,12/8/18,76/71/80,6/2/9}
\PReFlowTaskRow{task2}{4}{1/0/2,0/0/0,0/0/0,0/0/0,0/0/0,0/0/0,0/0/0,0/0/0,0/0/0,0/0/0,0/0/0,0/0/0,0/0/0,0/0/0,0/0/0,0/0/0,0/0/0,0/0/0,4/2/6,2/1/4}
\PReFlowTaskRow{task3}{55}{32/23/41,0/0/1,6/3/9,19/15/23,1/0/1,11/9/12,4/2/6,18/14/21,0/0/0,16/12/20,1/0/2,3/2/5,55/50/59,16/9/24,19/14/23,4/3/6,6/3/9,18/12/24,37/30/45,29/21/37}
\PReFlowTaskRow{task4}{47}{10/4/16,0/0/0,6/2/11,9/6/13,0/0/0,5/4/6,8/5/12,4/2/6,0/0/0,2/1/3,0/0/0,1/0/2,2/1/3,16/12/20,12/8/17,0/0/0,3/1/4,6/4/8,46/39/52,47/35/57}
\PReFlowTaskRow{task5}{40}{7/3/12,0/0/0,11/6/17,9/5/13,0/0/0,5/3/6,16/6/27,9/6/11,1/0/2,2/0/4,2/0/4,1/0/3,29/15/42,19/11/26,19/12/27,0/0/0,1/0/2,5/3/7,40/35/46,35/20/48}
\PReFlowTaskAgg{hl}{40}{17/15/19,0/0/0,5/4/8,9/7/11,0/0/0,5/4/5,6/3/8,6/5/8,0/0/0,6/5/7,3/2/5,3/2/3,24/21/27,11/9/14,12/9/14,2/1/3,2/1/3,8/7/10,40/38/43,24/20/28}
\PReFlowTaskGroup{scene}
\PReFlowTaskRow{task1}{100}{98/96/99,66/54/77,100/100/100,99/99/100,62/57/67,79/73/85,100/100/100,100/100/100,100/99/100,100/100/100,100/100/100,95/92/97,93/91/95,100/100/100,100/100/100,100/100/100,100/100/100,100/99/100,100/99/100,100/99/100}
\PReFlowTaskRow{task2}{100}{90/87/93,80/73/87,99/98/100,71/65/76,15/11/18,88/79/94,99/98/100,98/97/100,99/98/100,78/74/82,100/100/100,97/95/98,63/58/69,98/97/100,99/99/100,98/97/99,100/98/100,100/100/100,67/60/73,100/99/100}
\PReFlowTaskRow{task3}{100}{51/42/59,41/20/60,99/97/100,97/96/98,14/12/17,23/20/27,92/90/95,91/88/94,79/74/84,97/96/98,97/95/99,61/53/68,71/66/75,100/99/100,95/92/98,100/99/100,90/87/93,96/92/100,94/91/96,97/96/98}
\PReFlowTaskRow{task4}{100}{60/48/71,49/26/72,96/94/98,92/90/94,71/66/77,0/0/1,6/1/11,86/71/96,8/2/15,92/89/94,100/100/100,35/28/43,98/97/99,100/99/100,93/86/99,100/99/100,52/25/78,99/98/100,98/97/99,97/96/99}
\PReFlowTaskRow{task5}{100}{27/11/43,17/7/31,96/95/97,33/29/37,27/22/33,0/0/1,74/60/84,66/63/70,53/41/62,25/22/28,100/100/100,24/19/30,95/93/97,87/84/90,88/82/93,88/86/90,22/3/47,92/89/96,76/70/80,89/88/90}
\PReFlowTaskAgg{scene}{99}{65/61/69,50/43/57,98/97/99,78/77/80,38/35/41,38/36/40,74/72/76,88/85/91,68/65/70,78/77/80,99/99/100,62/60/65,84/83/85,97/96/98,95/92/97,97/96/98,72/66/80,97/96/98,87/85/89,97/96/97}
\PReFlowTaskGroup{p33}
\PReFlowTaskRow{task1}{100}{99/98/100,1/0/2,26/9/47,99/98/100,8/6/9,87/77/94,100/100/100,98/95/99,98/96/100,89/72/100,83/58/100,100/100/100,100/100/100,98/95/100,98/96/100,100/100/100,100/100/100,100/100/100,100/100/100,100/100/100}
\PReFlowTaskRow{task2}{100}{77/61/91,0/0/0,75/51/100,79/58/96,2/1/4,55/38/70,100/100/100,92/83/98,66/43/89,85/68/97,83/58/100,100/100/100,100/100/100,100/100/100,100/100/100,100/100/100,100/100/100,100/100/100,100/100/100,100/100/100}
\PReFlowTaskRow{task3}{100}{85/81/89,0/0/0,78/54/98,83/58/100,1/0/2,24/17/33,100/100/100,87/79/94,54/35/72,18/5/35,83/58/100,98/97/99,100/100/100,100/100/100,100/100/100,100/100/100,100/100/100,100/100/100,100/100/100,100/100/100}
\PReFlowTaskRow{task4}{100}{62/45/77,0/0/0,92/80/100,84/62/100,1/0/2,25/18/30,100/100/100,85/74/94,72/66/78,87/83/90,83/58/100,97/94/99,100/100/100,100/100/100,100/98/100,100/100/100,100/100/100,100/100/100,100/100/100,100/100/100}
\PReFlowTaskRow{task5}{100}{70/50/87,1/0/2,9/1/26,6/2/9,1/0/3,47/37/58,100/100/100,88/80/96,51/28/74,9/2/19,100/100/100,100/99/100,100/100/100,100/99/100,100/100/100,100/100/100,100/100/100,100/100/100,100/100/100,100/100/100}
\PReFlowTaskAgg{p33}{100}{79/73/84,0/0/1,56/48/64,70/60/78,3/2/3,48/40/55,100/100/100,90/83/96,68/62/75,57/52/63,87/82/92,99/98/100,100/100/100,100/99/100,99/99/100,100/100/100,100/100/100,100/100/100,100/100/100,100/100/100}
\PReFlowTaskGroup{p44}
\PReFlowTaskRow{task1}{85}{0/0/0,29/19/39,0/0/0,17/8/28,0/0/0,56/33/78,0/0/0,0/0/0,0/0/0,0/0/0,0/0/0,71/58/83,0/0/0,0/0/0,20/12/28,85/77/92,0/0/1,0/0/1,74/38/98,80/58/93}
\PReFlowTaskRow{task2}{50}{0/0/0,13/6/20,0/0/0,1/0/2,0/0/1,12/3/25,0/0/0,0/0/0,0/0/0,0/0/0,1/0/2,13/3/26,1/0/1,0/0/1,5/1/9,7/2/12,0/0/1,0/0/0,11/0/31,50/28/71}
\PReFlowTaskRow{task3}{82}{0/0/0,16/8/25,0/0/0,3/1/7,0/0/0,41/21/61,0/0/0,0/0/0,0/0/0,0/0/0,0/0/0,39/26/53,0/0/1,0/0/0,3/1/5,56/48/64,0/0/0,1/0/2,34/9/62,82/70/92}
\PReFlowTaskRow{task4}{65}{0/0/0,6/2/10,0/0/1,2/0/5,0/0/0,11/0/26,0/0/0,0/0/0,0/0/0,0/0/0,0/0/0,14/4/25,0/0/0,0/0/0,4/1/9,29/18/41,0/0/1,0/0/0,23/3/49,65/47/79}
\PReFlowTaskRow{task5}{54}{0/0/0,13/4/25,0/0/0,2/0/6,0/0/1,3/0/8,0/0/0,0/0/0,0/0/0,0/0/0,0/0/0,32/20/45,0/0/0,0/0/1,2/0/3,19/8/31,0/0/1,0/0/1,19/1/43,54/33/72}
\PReFlowTaskAgg{p44}{66}{0/0/0,15/12/19,0/0/0,5/3/7,0/0/0,24/16/33,0/0/0,0/0/0,0/0/0,0/0/0,0/0/0,34/30/37,0/0/0,0/0/0,6/5/8,39/35/43,0/0/0,0/0/1,32/21/43,66/58/74}
\PReFlowTaskGroup{c2}
\PReFlowTaskRow{task1}{96}{29/24/35,0/0/0,84/80/88,81/77/86,8/7/10,55/51/60,50/45/55,62/58/67,36/32/40,80/75/85,90/88/92,77/73/81,16/14/18,85/80/89,84/80/88,89/86/92,96/94/98,56/46/66,43/39/46,88/84/93}
\PReFlowTaskRow{task2}{87}{6/3/10,0/0/0,49/40/56,46/40/52,0/0/1,39/32/46,46/41/49,52/46/58,35/30/40,29/25/33,87/82/91,28/21/34,12/10/14,79/72/86,84/80/88,77/71/83,86/81/90,32/27/36,24/18/30,55/49/62}
\PReFlowTaskRow{task3}{85}{2/1/5,0/0/0,38/31/46,42/35/49,0/0/0,44/36/50,50/46/55,52/49/56,31/26/36,33/29/38,85/82/89,44/36/52,10/8/12,54/47/61,59/51/66,55/47/62,82/76/88,20/18/22,12/10/13,59/50/68}
\PReFlowTaskRow{task4}{35}{1/0/2,0/0/0,8/5/12,10/8/12,0/0/0,13/10/16,11/8/13,18/16/21,16/13/18,5/3/8,32/27/37,14/11/16,4/2/6,22/18/25,18/16/21,22/16/28,35/32/38,10/7/13,4/2/6,32/27/36}
\PReFlowTaskRow{task5}{83}{4/2/7,0/0/0,56/46/64,50/43/56,0/0/1,42/36/48,48/44/52,41/36/46,56/51/61,16/12/19,76/72/80,39/33/44,11/8/14,82/78/85,81/79/82,83/80/86,75/70/80,64/58/71,10/7/12,61/56/66}
\PReFlowTaskAgg{c2}{75}{9/8/10,0/0/0,47/44/50,46/43/49,2/2/2,38/36/41,41/39/43,45/43/47,35/33/36,33/31/34,74/72/76,40/37/43,11/10/12,64/62/66,65/63/67,65/63/68,75/73/76,37/34/39,18/17/20,59/56/62}
\PReFlowTaskGroup{c3}
\PReFlowTaskRow{task1}{40}{4/2/5,2/0/3,14/8/19,15/10/19,0/0/1,39/33/46,40/34/46,40/33/46,24/17/31,26/21/31,7/3/10,11/8/13,2/1/2,14/10/17,11/9/14,16/12/22,19/15/22,12/8/16,10/7/13,32/25/40}
\PReFlowTaskRow{task2}{2}{0/0/0,0/0/0,0/0/1,0/0/1,0/0/0,0/0/0,0/0/0,0/0/0,0/0/1,1/0/2,0/0/0,0/0/1,0/0/0,1/0/2,0/0/0,2/1/4,2/0/3,0/0/1,2/1/4,0/0/2}
\PReFlowTaskRow{task3}{4}{0/0/0,0/0/0,3/1/5,1/0/1,0/0/0,1/0/1,0/0/1,0/0/0,0/0/0,0/0/1,0/0/1,0/0/1,0/0/0,2/2/3,2/0/3,2/1/4,4/1/6,2/1/4,0/0/1,0/0/1}
\PReFlowTaskRow{task4}{2}{0/0/0,0/0/0,0/0/0,0/0/0,0/0/0,0/0/0,0/0/0,0/0/0,0/0/0,0/0/1,0/0/0,0/0/1,0/0/0,0/0/1,0/0/0,2/1/3,0/0/1,0/0/0,0/0/1,0/0/1}
\PReFlowTaskRow{task5}{2}{0/0/0,0/0/0,0/0/0,0/0/0,0/0/0,0/0/0,0/0/0,0/0/0,0/0/0,0/0/0,0/0/0,0/0/0,0/0/0,0/0/0,0/0/0,0/0/0,0/0/0,0/0/1,2/1/3,0/0/0}
\PReFlowTaskAgg{c3}{8}{1/0/1,0/0/1,3/2/5,3/2/4,0/0/0,8/7/9,8/7/9,8/7/9,5/3/6,6/5/6,1/1/2,2/2/3,0/0/0,3/3/4,3/2/3,5/4/6,5/4/6,3/2/4,3/2/4,7/5/8}
\PReFlowTaskGroup{c4}
\PReFlowTaskRow{task1}{90}{35/23/48,0/0/0,0/0/1,11/4/22,0/0/0,0/0/0,2/2/3,2/1/3,12/3/25,90/87/92,9/6/12,19/8/30,8/5/11,13/8/18,67/54/79,32/19/47,16/10/23,50/32/67,72/54/86,59/46/70}
\PReFlowTaskRow{task2}{85}{6/1/12,0/0/0,0/0/0,0/0/0,0/0/0,0/0/0,0/0/0,0/0/0,1/0/1,0/0/0,1/0/1,2/0/4,0/0/0,0/0/0,1/0/2,0/0/0,0/0/0,23/13/34,85/73/94,1/0/2}
\PReFlowTaskRow{task3}{70}{3/1/4,0/0/0,0/0/0,0/0/0,0/0/0,0/0/0,0/0/1,0/0/0,0/0/0,0/0/0,1/0/2,2/1/3,2/1/5,1/0/2,3/1/5,0/0/0,0/0/1,1/0/2,70/61/77,15/9/21}
\PReFlowTaskRow{task4}{19}{0/0/0,0/0/0,0/0/0,0/0/0,0/0/0,0/0/0,0/0/0,0/0/0,0/0/0,4/1/7,0/0/0,1/0/2,0/0/0,0/0/0,0/0/0,0/0/0,0/0/0,0/0/0,19/10/28,0/0/0}
\PReFlowTaskRow{task5}{0}{0/0/0,0/0/0,0/0/0,0/0/0,0/0/0,0/0/0,0/0/0,0/0/0,0/0/0,0/0/0,0/0/0,0/0/0,0/0/0,0/0/0,0/0/0,0/0/0,0/0/0,0/0/0,0/0/0,0/0/0}
\PReFlowTaskAgg{c4}{49}{9/6/11,0/0/0,0/0/0,2/1/5,0/0/0,0/0/0,1/0/1,0/0/1,3/1/5,19/18/19,2/2/3,5/3/7,2/1/3,3/2/4,14/11/17,6/4/9,3/2/5,15/11/19,49/44/53,15/12/17}
\bottomrule
\end{tabular}}
\endgroup
\caption{
\textbf{Per-task offline success rates (\%) after 1M gradient steps.}
Each cell shows the mean over seeds, with a 95\% confidence interval in small gray brackets above it.
\texttt{agg.} rows report the mean across the five tasks.
\best{Best results in each row are highlighted}, including ties.
}
\label{tab:task_offline}
\end{table}

\begin{table}[p]
\centering
\begingroup
\newcommand{\PReFlowTaskBest}[1]{{%
  \setlength{\fboxsep}{0.4pt}%
  \colorbox{background}{#1}}}
\renewcommand{\PReFlowTaskMean}[2]{%
  \def\PReFlowtmp{#2}\def\PReFlowdash{-}%
  \ifx\PReFlowtmp\PReFlowdash
    \text{--}%
  \else
    \ifnum#1=#2\relax
      \PReFlowTaskBest{#2}%
    \else
      #2%
    \fi
  \fi
}
\def\PReFlowTaskCell#1#2/#3/#4\PReFlowCellStop{%
  \if\relax\detokenize{#3}\relax
    \vphantom{\PReFlowTaskBest{0}}\PReFlowTaskMean{#1}{#2}%
  \else
    \shortstack[c]{%
      {\fontsize{5.3}{5.8}\selectfont\textcolor{gray}{[#3,#4]}}%
      \\[-1.4pt]\vphantom{\PReFlowTaskBest{0}}\PReFlowTaskMean{#1}{#2}}%
  \fi
}
\def\PReFlowTaskCells#1#2,#3\PReFlowListStop{%
  &\PReFlowTaskCell{#1}#2\PReFlowCellStop
  \if\relax\detokenize{#3}\relax
  \else
    \PReFlowTaskCells{#1}#3\PReFlowListStop
  \fi
}
\renewcommand{\PReFlowTaskRow}[3]{%
  \texttt{\textbf{#1}}\PReFlowTaskCells{#2}#3,\PReFlowListStop\\
}
\renewcommand{\PReFlowTaskAgg}[3]{%
  \texttt{\textbf{agg.}}\PReFlowTaskCells{#2}#3,\PReFlowListStop\\
}
\newcommand{\PReFlowTaskGroup}[1]{%
  \addlinespace[2pt]
  \multicolumn{21}{@{}l@{}}{%
    {\color{gray}\ttfamily\bfseries #1}}\\[-1pt]
}
\newcommand{\PReFlowTaskHead}[1]{%
  {\fontsize{7.0}{8.0}\selectfont\ttfamily\bfseries #1}%
}
\fontsize{8.0}{9.0}\selectfont
\setlength{\tabcolsep}{5.5pt}
\renewcommand{\arraystretch}{1.0}
\maxsizebox*{\linewidth}{0.86\textheight}{%
\begin{tabular}{@{}l*{20}{c}@{}}
\toprule
\PReFlowTaskHead{Task} & \PReFlowTaskHead{ReBRAC} & \PReFlowTaskHead{FBRAC} & \PReFlowTaskHead{BAM} & \PReFlowTaskHead{FQL} & \PReFlowTaskHead{FAWAC} & \PReFlowTaskHead{CGQL} & \PReFlowTaskHead{CGQL-M} & \PReFlowTaskHead{CGQL-L} & \PReFlowTaskHead{DAC} & \PReFlowTaskHead{QSM} & \PReFlowTaskHead{DSRL} & \PReFlowTaskHead{FEdit} & \PReFlowTaskHead{IFQL} & \PReFlowTaskHead{QAM} & \PReFlowTaskHead{QAM-F} & \PReFlowTaskHead{QAM-E} & \PReFlowTaskHead{TRQAM} & \PReFlowTaskHead{Q-Flow} & \PReFlowTaskHead{RQL} & \PReFlowTaskHead{PReFlow} \\
\midrule
\PReFlowTaskGroup{al}
\PReFlowTaskRow{task1}{100}{98/98/99,95/87/100,97/96/98,100/99/100,2/0/5,99/98/100,90/87/92,84/81/86,96/94/97,97/94/99,96/94/97,97/96/98,65/45/82,98/97/99,96/94/97,98/96/99,98/96/99,99/98/100,86/82/88,99/98/100}
\PReFlowTaskRow{task2}{97}{93/90/95,72/48/94,68/62/73,97/95/98,1/0/1,83/80/86,46/41/50,38/34/42,77/71/82,73/61/84,70/63/76,89/87/91,74/70/78,94/92/96,78/74/82,90/88/92,89/87/91,94/91/96,82/75/89,95/93/96}
\PReFlowTaskRow{task3}{100}{100/99/100,91/74/100,99/98/100,100/100/100,61/56/66,98/97/100,96/95/98,94/89/97,100/99/100,100/100/100,98/98/99,100/100/100,87/84/89,100/100/100,99/98/100,100/99/100,100/99/100,100/100/100,100/99/100,100/100/100}
\PReFlowTaskRow{task4}{100}{98/98/99,92/77/100,94/92/96,100/99/100,20/18/22,96/94/98,89/88/90,87/83/90,98/97/99,99/99/100,90/87/92,96/95/97,83/80/86,97/96/99,94/92/97,96/94/97,99/98/100,99/98/100,94/92/96,98/96/99}
\PReFlowTaskRow{task5}{100}{99/98/100,98/96/99,96/95/98,100/99/100,23/18/28,96/95/98,90/87/92,90/88/93,99/98/100,99/99/100,94/92/96,97/96/98,73/58/83,99/99/100,97/96/98,97/95/99,99/98/100,99/97/100,88/85/90,99/98/100}
\PReFlowTaskAgg{al}{99}{98/97/98,89/83/95,91/90/92,99/99/99,21/20/23,94/94/95,82/81/83,79/77/80,94/93/95,94/91/96,90/88/91,96/95/96,76/71/81,98/97/98,93/92/94,96/95/97,97/96/97,98/97/99,90/88/91,98/98/99}
\PReFlowTaskGroup{ag}
\PReFlowTaskRow{task1}{96}{80/64/89,64/39/86,25/7/46,89/72/99,0/0/0,92/90/94,2/0/3,34/24/46,54/38/69,39/21/59,25/13/38,90/86/93,0/0/0,14/10/20,70/60/78,89/72/98,90/86/95,33/27/38,23/12/34,96/94/98}
\PReFlowTaskRow{task2}{98}{96/94/98,79/56/96,0/0/0,97/96/99,0/0/0,92/91/94,88/86/90,74/70/79,95/94/96,98/98/99,16/5/29,87/76/95,6/3/9,76/72/79,90/88/93,98/98/99,96/95/98,59/44/73,96/94/98,97/95/99}
\PReFlowTaskRow{task3}{98}{19/10/28,77/55/94,0/0/0,87/70/97,0/0/0,55/40/66,14/6/25,2/0/3,94/90/96,98/97/99,1/0/1,65/53/77,0/0/0,3/2/5,9/7/11,93/86/97,82/75/88,6/4/8,50/40/60,83/64/95}
\PReFlowTaskRow{task4}{100}{81/62/96,88/73/96,14/0/35,94/84/99,0/0/0,94/93/96,90/83/95,87/83/90,96/94/98,100/100/100,66/51/77,93/89/96,2/0/4,75/72/77,91/89/93,98/96/99,93/91/95,79/75/84,26/8/45,97/96/98}
\PReFlowTaskRow{task5}{99}{97/95/98,95/92/97,17/0/42,98/96/99,0/0/0,93/91/95,93/90/95,87/84/89,99/98/100,99/98/100,93/90/96,97/96/98,79/75/82,97/95/98,98/98/99,96/94/98,98/97/100,95/93/97,97/95/99,99/98/100}
\PReFlowTaskAgg{ag}{95}{75/69/79,81/72/88,11/5/18,93/88/97,0/0/0,85/82/88,57/55/60,57/54/59,87/84/91,87/83/91,40/36/45,86/83/90,17/16/18,53/52/54,72/70/74,95/91/97,92/90/94,54/51/58,58/54/63,94/91/97}
\PReFlowTaskGroup{hm}
\PReFlowTaskRow{task1}{98}{14/4/29,27/22/33,53/48/58,74/62/86,22/17/28,44/40/49,80/60/96,82/74/88,87/80/93,84/78/88,68/49/84,75/55/92,39/30/48,41/22/60,74/56/88,63/39/85,73/59/82,98/96/99,98/97/100,84/63/99}
\PReFlowTaskRow{task2}{100}{98/96/99,91/89/93,93/91/95,99/98/100,57/52/62,88/85/92,100/99/100,96/94/98,97/96/98,98/97/99,81/69/91,81/58/99,56/46/66,99/98/99,100/100/100,100/99/100,94/91/97,99/98/100,100/99/100,97/92/100}
\PReFlowTaskRow{task3}{100}{74/50/99,44/35/53,86/84/89,98/96/100,24/20/26,92/91/94,88/69/100,97/96/98,99/98/100,99/98/100,85/76/92,79/56/98,42/32/54,94/83/100,97/92/100,91/75/100,74/66/81,99/98/100,99/98/100,100/99/100}
\PReFlowTaskRow{task4}{98}{4/0/9,6/3/10,25/22/28,65/46/81,4/2/5,40/33/46,90/73/99,40/35/44,56/47/64,19/12/26,53/38/68,68/44/89,51/44/58,82/66/93,57/36/76,60/37/82,44/40/50,83/79/86,98/97/99,97/91/100}
\PReFlowTaskRow{task5}{100}{99/98/100,78/63/89,97/95/99,98/95/100,50/42/57,92/91/94,100/100/100,99/98/100,99/98/100,99/98/100,96/94/98,100/100/100,69/57/80,99/98/100,100/99/100,100/100/100,97/95/99,100/99/100,100/99/100,100/99/100}
\PReFlowTaskAgg{hm}{99}{58/52/63,49/46/53,71/70/72,87/82/91,31/29/34,71/70/73,92/86/97,83/81/84,88/86/90,80/78/81,77/71/82,80/72/88,51/47/56,83/78/88,85/80/90,83/75/90,76/73/79,96/95/96,99/98/99,95/91/99}
\PReFlowTaskGroup{hl}
\PReFlowTaskRow{task1}{83}{44/26/62,4/2/6,0/0/0,11/6/17,0/0/0,1/0/2,17/5/32,9/5/13,25/12/39,0/0/0,45/27/62,19/12/28,1/0/2,4/1/6,13/6/22,7/1/17,4/1/7,17/14/20,60/54/66,83/76/88}
\PReFlowTaskRow{task2}{25}{1/0/2,0/0/0,0/0/1,0/0/0,0/0/0,0/0/0,0/0/0,0/0/0,0/0/0,0/0/0,25/13/37,0/0/0,0/0/1,0/0/0,0/0/0,4/0/12,0/0/1,0/0/0,4/2/8,10/2/18}
\PReFlowTaskRow{task3}{92}{44/28/58,16/12/18,27/10/47,41/36/48,0/0/1,12/9/15,67/51/79,20/17/23,72/65/78,10/7/13,39/21/58,8/1/17,25/16/34,41/28/52,44/34/53,57/35/76,19/13/26,24/18/27,92/89/94,88/65/99}
\PReFlowTaskRow{task4}{65}{20/12/30,0/0/0,20/10/29,15/10/21,0/0/0,7/5/9,44/32/55,7/5/9,43/31/53,0/0/0,20/5/38,10/3/17,7/4/10,31/19/42,22/14/30,6/0/18,6/2/11,7/4/9,50/42/58,65/34/88}
\PReFlowTaskRow{task5}{80}{7/1/14,0/0/0,22/13/32,10/6/16,0/0/0,4/3/6,40/24/56,7/4/10,41/30/52,0/0/0,16/2/32,3/0/6,4/2/6,27/22/32,20/11/28,8/0/23,4/1/10,5/3/6,47/42/52,80/77/84}
\PReFlowTaskAgg{hl}{65}{23/18/28,4/3/5,14/10/19,16/13/18,0/0/0,5/4/6,34/28/39,8/7/10,36/32/40,2/1/3,29/22/36,8/5/11,7/5/9,21/17/24,20/16/23,16/11/22,7/5/9,10/9/12,51/48/53,65/58/72}
\PReFlowTaskGroup{scene}
\PReFlowTaskRow{task1}{100}{100/100/100,100/100/100,100/100/100,100/100/100,100/100/100,100/99/100,100/100/100,100/100/100,100/100/100,100/100/100,100/100/100,100/100/100,100/100/100,100/100/100,100/100/100,100/100/100,100/100/100,100/100/100,100/100/100,100/100/100}
\PReFlowTaskRow{task2}{100}{100/100/100,100/100/100,94/84/100,100/100/100,100/99/100,99/99/100,99/98/100,98/96/99,100/99/100,100/100/100,100/100/100,100/100/100,100/99/100,100/100/100,100/100/100,100/100/100,100/100/100,100/100/100,100/100/100,100/100/100}
\PReFlowTaskRow{task3}{100}{99/98/100,100/100/100,100/100/100,100/99/100,69/50/86,93/90/96,98/97/99,77/73/81,100/99/100,100/100/100,100/100/100,100/100/100,100/99/100,100/100/100,100/100/100,100/100/100,100/100/100,100/99/100,100/100/100,100/100/100}
\PReFlowTaskRow{task4}{100}{99/98/100,99/98/100,98/97/99,100/99/100,100/99/100,72/68/75,90/86/93,60/54/66,94/91/97,98/96/99,100/100/100,98/97/99,98/97/99,100/99/100,99/98/100,99/98/100,99/97/100,98/97/100,89/84/93,98/97/99}
\PReFlowTaskRow{task5}{100}{96/94/98,94/89/98,100/99/100,58/52/64,95/92/97,60/51/68,75/69/80,2/1/4,93/91/95,12/8/15,100/99/100,95/93/97,85/82/88,100/100/100,99/99/100,100/99/100,98/96/100,98/97/100,99/98/100,96/94/98}
\PReFlowTaskAgg{scene}{100}{99/98/99,99/98/100,98/96/100,92/90/93,93/89/96,85/83/87,92/91/94,67/66/69,97/96/98,82/81/83,100/100/100,99/98/99,96/96/97,100/100/100,100/99/100,100/100/100,99/99/100,99/99/100,98/97/98,99/98/99}
\PReFlowTaskGroup{p33}
\PReFlowTaskRow{task1}{100}{100/100/100,100/100/100,100/100/100,100/100/100,100/99/100,100/100/100,100/100/100,100/100/100,91/86/95,99/98/100,83/58/100,100/100/100,100/100/100,100/100/100,100/100/100,100/100/100,100/100/100,100/100/100,100/100/100,100/100/100}
\PReFlowTaskRow{task2}{100}{100/100/100,92/75/100,92/75/100,100/100/100,4/2/5,100/100/100,100/100/100,100/100/100,42/32/54,99/98/100,83/58/100,100/100/100,100/100/100,100/100/100,100/100/100,100/100/100,100/100/100,100/100/100,100/100/100,100/100/100}
\PReFlowTaskRow{task3}{100}{100/100/100,100/100/100,100/100/100,100/100/100,1/0/2,100/100/100,100/99/100,100/100/100,15/10/21,72/57/84,83/58/100,100/100/100,100/99/100,100/100/100,100/100/100,100/100/100,100/100/100,100/100/100,100/100/100,100/100/100}
\PReFlowTaskRow{task4}{100}{100/100/100,100/100/100,84/59/100,100/100/100,0/0/1,100/100/100,88/85/91,100/100/100,21/15/29,84/77/90,83/58/100,100/100/100,98/97/100,100/100/100,100/100/100,100/100/100,100/100/100,100/100/100,100/100/100,100/100/100}
\PReFlowTaskRow{task5}{100}{100/100/100,84/62/100,34/10/59,100/100/100,1/0/2,100/100/100,100/100/100,100/100/100,21/14/28,77/65/88,100/100/100,100/100/100,100/100/100,100/100/100,100/100/100,100/100/100,100/100/100,100/100/100,100/100/100,100/100/100}
\PReFlowTaskAgg{p33}{100}{100/100/100,95/90/100,82/74/88,100/100/100,21/21/21,100/100/100,97/97/98,100/100/100,38/35/42,86/82/90,87/78/95,100/100/100,100/99/100,100/100/100,100/100/100,100/100/100,100/100/100,100/100/100,100/100/100,100/100/100}
\PReFlowTaskGroup{p44}
\PReFlowTaskRow{task1}{100}{0/0/0,100/100/100,0/0/0,100/100/100,0/0/1,100/100/100,42/20/64,39/19/60,0/0/0,100/99/100,0/0/0,100/100/100,0/0/0,0/0/0,100/100/100,100/100/100,0/0/1,0/0/1,100/100/100,100/100/100}
\PReFlowTaskRow{task2}{100}{0/0/0,100/100/100,0/0/0,100/100/100,0/0/1,92/75/100,17/1/38,35/12/60,0/0/0,99/98/100,0/0/0,100/100/100,0/0/0,8/0/25,100/100/100,100/100/100,0/0/0,0/0/1,100/100/100,100/100/100}
\PReFlowTaskRow{task3}{100}{0/0/0,83/58/100,0/0/0,100/100/100,0/0/1,100/100/100,37/15/63,14/0/33,0/0/0,100/99/100,0/0/0,100/100/100,0/0/0,0/0/1,100/100/100,94/84/100,0/0/0,0/0/1,100/100/100,100/100/100}
\PReFlowTaskRow{task4}{100}{0/0/0,92/75/100,0/0/0,100/100/100,0/0/0,91/74/100,16/0/35,23/5/44,0/0/0,99/98/100,0/0/0,100/100/100,0/0/0,0/0/0,100/100/100,100/100/100,0/0/0,0/0/1,100/100/100,100/100/100}
\PReFlowTaskRow{task5}{100}{0/0/0,83/58/100,0/0/0,100/100/100,0/0/1,92/75/100,32/10/56,13/0/30,0/0/0,91/74/100,0/0/0,100/100/100,0/0/0,8/0/25,100/100/100,100/100/100,0/0/1,0/0/1,100/100/100,100/100/100}
\PReFlowTaskAgg{p44}{100}{0/0/0,92/85/98,0/0/0,100/100/100,0/0/1,95/88/100,29/20/38,25/16/34,0/0/0,98/94/100,0/0/0,100/100/100,0/0/0,3/0/8,100/100/100,99/97/100,0/0/0,0/0/1,100/100/100,100/100/100}
\PReFlowTaskGroup{c2}
\PReFlowTaskRow{task1}{100}{99/98/100,97/94/99,100/99/100,100/100/100,10/4/18,100/100/100,100/100/100,100/100/100,100/100/100,100/100/100,100/99/100,100/100/100,100/100/100,100/99/100,100/100/100,100/99/100,100/100/100,100/99/100,99/98/100,100/100/100}
\PReFlowTaskRow{task2}{100}{99/98/99,96/93/98,99/99/100,100/99/100,2/1/4,99/98/100,98/97/99,100/99/100,100/100/100,99/99/100,98/97/100,100/100/100,97/95/99,100/100/100,100/100/100,100/100/100,100/100/100,100/99/100,98/97/99,100/99/100}
\PReFlowTaskRow{task3}{100}{98/97/99,96/93/98,99/98/100,100/99/100,2/1/2,99/99/100,98/97/100,99/98/100,100/99/100,99/99/100,100/99/100,100/100/100,97/96/98,100/100/100,100/100/100,100/100/100,100/100/100,100/99/100,98/96/99,100/100/100}
\PReFlowTaskRow{task4}{100}{95/92/98,68/56/80,97/95/99,100/100/100,0/0/0,87/77/93,93/90/96,97/95/99,97/96/98,99/98/100,98/96/99,100/99/100,6/4/8,100/99/100,100/99/100,100/100/100,100/99/100,98/98/99,60/51/68,100/100/100}
\PReFlowTaskRow{task5}{100}{94/92/96,41/26/56,97/95/99,100/100/100,15/8/23,99/98/100,98/96/99,99/98/100,99/98/100,94/92/96,99/98/100,100/100/100,97/95/99,100/100/100,100/100/100,100/99/100,100/100/100,100/98/100,99/98/100,100/100/100}
\PReFlowTaskAgg{c2}{100}{97/96/98,79/75/83,98/98/99,100/100/100,6/4/8,97/95/98,98/97/98,99/99/99,99/99/99,98/98/99,99/99/99,100/100/100,79/79/80,100/100/100,100/100/100,100/100/100,100/100/100,99/99/100,91/89/93,100/100/100}
\PReFlowTaskGroup{c3}
\PReFlowTaskRow{task1}{100}{61/37/84,20/4/40,83/58/100,100/99/100,0/0/0,100/99/100,100/100/100,99/99/100,100/100/100,99/99/100,1/1/2,85/64/100,0/0/1,100/100/100,98/95/100,100/100/100,100/100/100,47/19/78,90/87/94,100/100/100}
\PReFlowTaskRow{task2}{98}{9/0/24,1/0/2,71/49/91,61/38/84,0/0/0,0/0/0,0/0/1,1/0/2,8/4/13,14/8/20,0/0/0,2/0/5,0/0/0,67/43/89,65/42/86,98/97/99,90/78/96,28/6/52,34/15/52,96/94/97}
\PReFlowTaskRow{task3}{99}{0/0/1,1/0/1,57/44/66,81/78/84,0/0/0,0/0/0,0/0/0,0/0/0,9/5/12,11/5/17,0/0/0,81/77/85,0/0/0,96/94/98,88/85/92,99/98/100,93/91/96,36/22/52,36/21/50,92/88/96}
\PReFlowTaskRow{task4}{99}{0/0/1,0/0/1,50/29/70,56/42/68,0/0/0,0/0/0,0/0/0,1/0/2,7/5/10,9/7/11,0/0/0,31/13/51,0/0/0,88/71/98,89/85/94,99/99/100,94/89/98,15/6/28,4/0/11,89/86/93}
\PReFlowTaskRow{task5}{26}{0/0/0,0/0/0,0/0/0,0/0/0,0/0/0,0/0/0,0/0/0,0/0/0,0/0/0,0/0/0,0/0/0,0/0/0,0/0/0,0/0/0,0/0/1,0/0/0,1/0/3,0/0/2,26/17/34,0/0/0}
\PReFlowTaskAgg{c3}{79}{14/9/20,5/1/9,52/44/60,60/54/65,0/0/0,20/20/20,20/20/20,20/20/20,25/24/26,27/25/28,0/0/0,40/34/45,0/0/0,70/64/76,68/64/73,79/79/79,76/73/78,25/17/34,38/33/43,75/74/77}
\PReFlowTaskGroup{c4}
\PReFlowTaskRow{task1}{100}{100/100/100,98/96/100,100/100/100,100/100/100,0/0/0,72/52/86,81/62/96,96/94/98,100/100/100,100/100/100,100/100/100,100/100/100,100/99/100,100/100/100,100/100/100,100/100/100,100/100/100,100/100/100,100/100/100,100/100/100}
\PReFlowTaskRow{task2}{100}{100/100/100,59/35/80,85/78/90,100/100/100,0/0/0,0/0/1,52/37/66,66/53/77,62/52/70,50/29/70,97/96/98,99/98/100,85/80/91,80/64/90,99/98/100,100/99/100,97/95/99,100/99/100,100/100/100,100/100/100}
\PReFlowTaskRow{task3}{100}{100/99/100,48/25/70,85/77/92,99/98/100,0/0/0,0/0/0,43/32/52,48/36/58,78/72/83,92/88/95,95/93/96,99/98/100,64/58/70,93/90/96,97/96/99,97/95/99,93/92/94,100/99/100,99/98/100,100/100/100}
\PReFlowTaskRow{task4}{100}{98/94/100,51/31/72,50/24/75,99/99/100,0/0/0,0/0/0,0/0/0,0/0/1,57/49/65,93/89/96,97/96/98,100/100/100,54/46/63,82/72/90,99/99/100,99/98/100,92/89/96,100/98/100,100/99/100,100/100/100}
\PReFlowTaskRow{task5}{1}{0/0/0,0/0/0,0/0/0,0/0/0,0/0/0,0/0/0,0/0/0,0/0/0,0/0/0,0/0/0,1/0/2,0/0/0,0/0/0,0/0/0,0/0/0,0/0/0,0/0/0,0/0/0,0/0/0,0/0/0}
\PReFlowTaskAgg{c4}{80}{80/79/80,51/44/59,64/59/69,80/79/80,0/0/0,14/10/17,35/30/40,42/39/45,59/56/62,67/62/71,78/77/78,80/79/80,61/58/63,71/67/74,79/79/80,79/79/80,76/76/77,80/80/80,80/79/80,80/80/80}
\bottomrule
\end{tabular}}
\endgroup
\caption{
\textbf{Per-task performance after 500K online environment steps
(success rate, \%).}
Data sources, seed counts, confidence intervals, and highlighting
follow Table~\ref{tab:task_offline}.
\best{Best results in each row are highlighted}, including ties.
}
\label{tab:task_online}
\end{table} 

\FloatBarrier
\begin{figure}[p]
    \centering
    \includegraphics[width=0.96\linewidth,height=0.88\textheight,keepaspectratio]{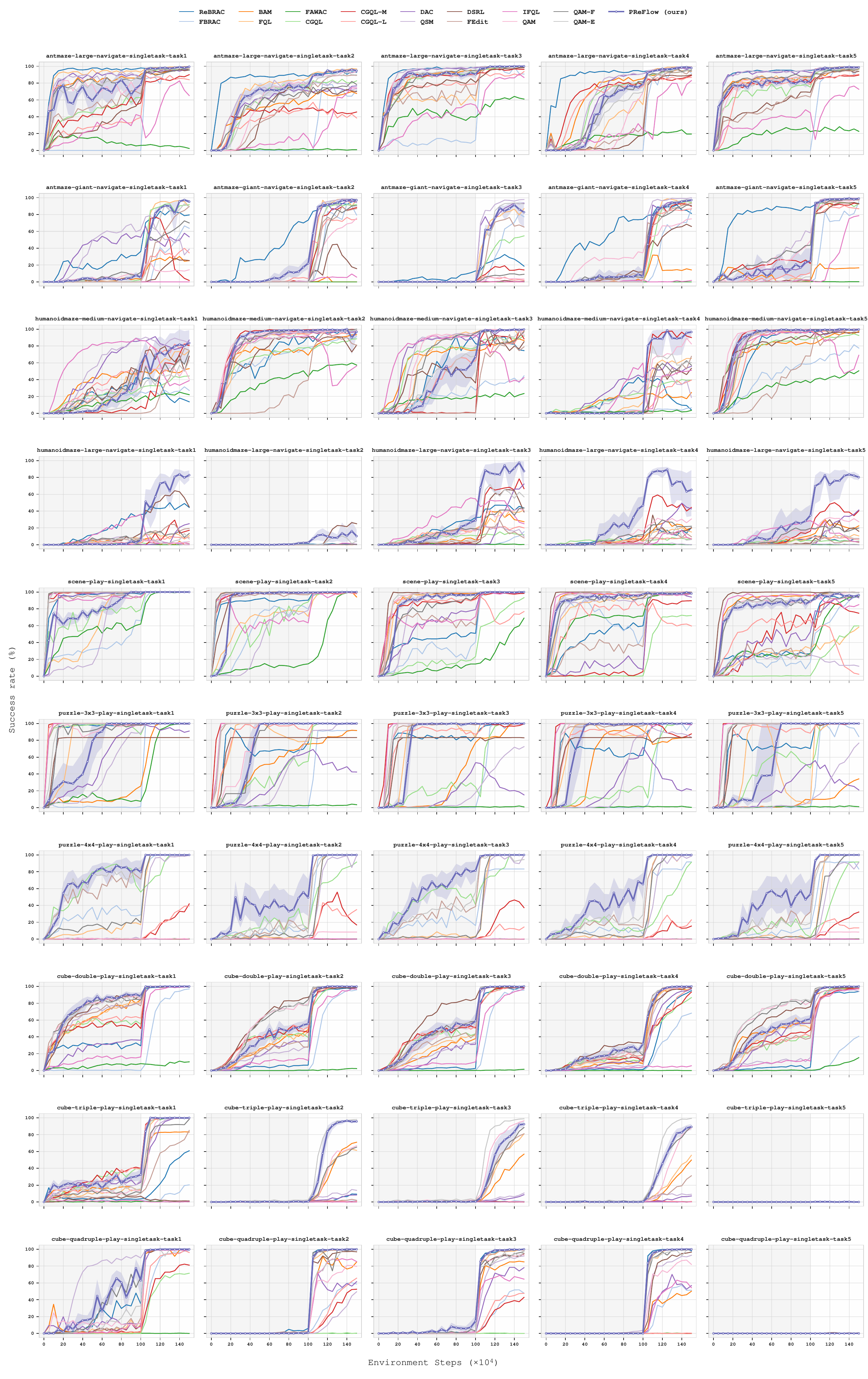}
    \caption{
\textbf{Full training curves on the 50 OGBench tasks.}
Online fine-tuning begins after 1M offline gradient steps and
continues for 500K environment steps.
}
    \label{fig:full_training_curves}
\end{figure}